%% file: paper.tex
\documentclass[twoside]{article}

\usepackage[accepted]{aistats2021}
\usepackage[round]{natbib}

\input{macros}

\begin{document}



\twocolumn[

\aistatstitle{Algorithms for Fairness in Sequential Decision Making}

\aistatsauthor{ Min Wen \And Osbert Bastani \And  Ufuk Topcu }

\aistatsaddress{ University of Pennsylvania \And University of Pennsylvania \And University of Texas at Austin } ]

\begin{abstract}
\input{abstract}
\end{abstract}

\input{main}

\bibliography{paper}
\bibliographystyle{plainnat}

\clearpage
\appendix
\input{appendix}

\end{document}

%% file: macros.tex
\usepackage{amsmath,amssymb,amsthm,graphicx}
\usepackage{algorithm,algpseudocode}

\newtheorem{theorem}{Theorem}[section]
\newtheorem{remark}[theorem]{Remark}
\newtheorem{lemma}[theorem]{Lemma}

\newtheorem{definition}[theorem]{Definition}

%% file: abstract.tex
It has recently been shown that if feedback effects of decisions are ignored, then imposing fairness constraints such as demographic parity or equality of opportunity can actually exacerbate unfairness. We propose to address this challenge by modeling feedback effects as Markov decision processes (MDPs). First, we propose analogs of fairness properties for the MDP setting. Second, we propose algorithms for learning fair decision-making policies for MDPs. Finally, we demonstrate the need to account for dynamical effects using simulations on a loan applicant MDP.

%% file: main.tex
\section{Introduction}

Machine learning has the potential to substantially improve performance in tasks such as legal and financial decision-making. However, biases in the data can be reflected in a decision-making policy trained on that data~\citep{dwork2012fairness}, which can result in decisions that unfairly discriminate against minorities. For example, consider the problem of deciding whether to give loans to applicants~\citep{hardt2016equality}. If minorities are historically given loans less frequently, then there may be less data on how reliably they repay loans. Thus, a learned policy may unfairly label minorities as higher risk and deny them loans.

So far, work on fairness has largely focused on supervised learning. However, na\"{i}vely imposing fairness constraints while ignoring even one-step feedback effects can actually harm minorities~\citep{liu2018delayed,creager2019causal,d2020fairness}. Thus, we must extend existing definitions of fairness to account for the feedback effects of the decisions being made on population members. For example, denying loans to individuals may have consequences on their financial security that need to be taken into account.

This paper proposes algorithms for learning fair decision-making policies that account for feedback effects of decisions. We model these effects as the dynamics of a Markov decision process (MDP), and extend existing fairness definitions to decision-making policies for a known MDP. We distinguish the quality of outcomes for the decision-maker (e.g., the bank) from the quality of the outcomes for individuals (e.g., a loan applicant). Then, fairness properties are constraints on the average quality of outcomes for individuals in different subpopulations (e.g., majorities and minorities are offered loans at the same frequency), whereas the reward measures the quality of outcomes for the decision-maker (e.g., the bank's profit). The key challenge is that learning with a fairness constraint is much more challenging in the MDP setting due to the inherent non-convexity. Building on work on constrained MDPs~\citep{altman1999constrained,wen2018constrained}, we propose novel algorithms for learning policies that satisfy fairness constraints. In particular, we propose two algorithms. First, we propose a model-based algorithm based that has optimality guarantees, but is limited to MDPs with finite state and action spaces and satisfies a \emph{separability} assumption saying that the sensitive attribute does not change over time. Second, we propose a model-free algorithm that is very general, but may not find the optimal policy.

We compare to two baselines that ignore dynamics: (i) an algorithm that optimistically pretends actions do not affect the state distribution (i.e., supervised learning), and (ii) an algorithm that conservatively assumes the state distribution can change adversarially on each step. In a simulation study on a loan applicant MDP based on~\citep{hardt2016equality}, we show that compared to our algorithm, the optimistic algorithm learns unfair policies, and the conservative algorithm learns fair but poorly performing policies. Our results demonstrate the importance of accounting for dynamics.\footnote{Our code is at: \url{https://github.com/wmgithub/fairness}.}

\textbf{Related work.}
For supervised learning, there have been several definitions of fairness, including demographic parity (i.e., members of the majority and miniority subpopulations have equal outcomes on average)~\citep{calders2009building}, equality of opportunity (i.e., \emph{qualified} members have equal outcomes on average)~\citep{hardt2016equality}, individual fairness~\citep{dwork2012fairness}, and causal fairness (i.e., protected attributes should not influence outcomes)~\citep{kusner2017counterfactual,kilbertus2017avoiding,nabi2018fair}. The appropriate definition depends on the application.

There has been recent interest in fairness for sequential decision making. For instance, ~\cite{liu2018delayed} has studied one-step feedback effects, ~\cite{creager2019causal} studies the impact of dynamics on fairness via simulations, and ~\citep{d2020fairness} proposes tools from causal inference to study fairness with dynamics. However, none of these approaches propose learning algorithms. For instance, the model in \cite{liu2018delayed} is highly stylized (e.g., they only consider a single time step) since their goal is to demonstrate the necessity of accounting for sequential decisions rather than study the general problem of algorithms for ensuring fairness in sequential decision-making.

In the case of unknown dynamics, there has been work in the bandit setting~\citep{joseph2016fairness,hashimoto2018fairness} and the MDP setting~\citep{jabbari2017fairness,elzayn2018fair}. However, they focus on fairness constraints for which the optimal policy is always fair, so solving for the optimal fair policy is trivial once the dynamics are known. In contrast, we are interested in the setting where fairness constraint is nontrivial even when the dynamics are known. There has been recent work studying fairness constraints~\citep{bechavod2019equal,kilbertus2019fair} in the setting of selective labels~\citep{lakkaraju2017selective}; however, there is no state in their setting. In addition, \cite{awasthi2020beyond} study how fairness definitions can be updated over time based on feedback; in their model, individuals do not recur across time steps as they do in ours.

There has been work on constrained MDPs~\citep{altman1999constrained,achiam2017constrained,wen2018constrained}. However, these approaches focus on constraints that bound some state-dependent cost function; in contrast, fairness constraints say that statistics of different groups must be equalized in some way.

\section{Fairness Constraints for MDPs}
\label{sec:problem}

\textbf{Preliminaries.}
A \emph{Markov decision process (MDP)} is a tuple $M=(S,A,D,P,R,\gamma)$, where $S=[n]=\{1,...,n\}$ are the states, $A=[m]$ are the actions, $D\in\mathbb{R}^{|S|}$ is the initial state distribution (i.e., $D_s$ is the probability of starting in state $s$), $P\in\mathbb{R}^{|S|\times|A|\times|S|}$ are the transitions (i.e., $P_{s,a,s'}$ is the probability of transitioning from $s$ to $s'$ taking action $a$), $R\in\mathbb{R}^{|S|\times|A|}$ are the rewards (i.e., $R_{s,a}$ is the reward obtained taking action $a$ in state $s$), and $\gamma\in\mathbb{R}$ is the discount factor. Let $\pi\in\mathbb{R}^{|S|\times|A|}$ be a stochastic policy (i.e., $\pi_{s,a}$ is the probability of taking action $a$ in state $s$). The induced transtions are $P^{(\pi)}\in\mathbb{R}^{|S|\times|S|}$, where $P^{(\pi)}_{s,s'}=\sum_{a\in A}\pi_{s,a}P_{s,a,s'}$. The time-discounted state distribution is
\begin{align*}
D^{(\pi)}&=(1-\gamma)\sum_{t=0}^{\infty}\gamma^tD^{(\pi,t)}
\end{align*}
where
\begin{align*}
D^{(\pi,t)}=\begin{cases}D&\text{if}~t=0\\P^{(\pi)}D^{(\pi,t-1)}&\text{otherwise},\end{cases}
\end{align*}
and the time-discounted state-action distribution is $\Lambda\in\mathbb{R}^{|S|\times|A|}$, where $\Lambda_{s,a}^{(\pi)}=D_s^{(\pi)}\pi_{s,a}$. Note that $\sum_a\pi_{s,a}=1$ and $\sum_sD_s^{(\pi)}=1$, so $\sum_{s,a}\Lambda_{s,a}=1$. The cumulative expected reward is
\begin{align*}
R^{(\pi)}
=(1-\gamma)\sum_{t=0}^{\infty}\gamma^t\langle R,\Lambda^{(\pi,t)}\rangle
=\mathbb{E}_{(s,a)\sim\Lambda^{(\pi)}}[R_{s,a}],
\end{align*}
where $\langle X,Y\rangle=\sum_{s\in S}\sum_{a\in A}X_{s,a}Y_{s,a}$; we include a normalizing constant of $1-\gamma$ to simplify notation, which does not affect the reinforcement learning problem since $R^{(\pi)}$ is scaled equally for different policies. Given policy class $\Pi$, the optimal policy is $\pi^*=\operatorname*{\arg\max}_{\pi\in\Pi}R^{(\pi)}$.

\textbf{Fairness.}
Consider a population of individuals (e.g., loan applicants) interacting with a decision-maker (e.g., a bank). States $S$ encode an individual's features (e.g., probability of repaying), actions $A$ are interventions (e.g., loan offer), and transitions $P$ encode state changes (e.g., changes in ability to repay). The decision-maker rewards are not always aligned with individual rewards, so we use rewards $R$ to indicate quality of outcomes for the decision-maker (e.g., the bank's profit), and \emph{individual rewards} $\rho\in\mathbb{R}^{|S|\times|A|}$ to indicate quality of outcomes for an individual (e.g., whether a loan is offered). The cumulative expected individual rewards is $\rho^{(\pi)}=\mathbb{E}_{(s,a)\sim\Lambda^{(\pi)}}[\rho_{s,a}]$.

Our goal is to learn the optimal policy for the decision-maker under a fairness constraint on the individual rewards. In particular, we want to ensure that $\pi$ does not favor the \emph{majority subpopulation} over the \emph{minority subpopulation}. The specific fairness constraint that should be used depends on the problem domain. We show how two constraints from the supervised learning setting can be extended to the MDP setting; as we discuss below, our results are more general.

First, we have the following extension of demographic parity to the MDP setting:
\begin{definition}
\rm
Let $\epsilon\in\mathbb{R}_+$, $M$ be an MDP with states $S=Z\times\tilde{S}$, where $Z=\{\text{maj},\text{min}\}$, and $\rho\in\mathbb{R}^{|S|\times|A|}$ be the individual rewards. For $z\in Z$, let
\begin{align*}
\Lambda_z^{(\pi)}=\Lambda^{(\pi)}\mid\exists\tilde{s}\in\tilde{S}~.~s_0=(z,\tilde{s})
\end{align*}
be the time-discounted state-action distribution conditioned on starting from an initial state $s_0$ in subpopulation $z$---i.e., $s_0$ has the form $s_0=(z,\tilde{s}_0)$ for some $\tilde{s}_0\in\tilde{S}$. More precisely,
\begin{align*}
(\Lambda_z^{(\pi)})_{s,a}&=(D_z^{(\pi)})_s\pi_{s,a}\hspace{0.1in}(\forall s\in S,a\in A) \\
(D_z^{(\pi)})_s&=(1-\gamma)\sum_{t=0}^{\infty}(\gamma P^{(\pi)})^tD_z\hspace{0.1in}(\forall s\in S) \\
(D_z)_{s_0}&=\nu^{-1}\cdot D_{s_0}\cdot\mathbb{I}[\exists\tilde{s}\in\tilde{S}~.~s_0=(z,\tilde{s})]\hspace{0.1in}(\forall s_0\in S),
\end{align*}
where $\nu$ is a normalizing constant. Furthermore, let $\rho^{(\pi)}$ conditioned on starting in subpopulation $z$ is $\rho_z^{(\pi)}=\mathbb{E}_{(s,a)\sim\Lambda_z^{(\pi)}}[\rho_{s,a}]$. Then, we say a policy $\pi$ satisfies $\epsilon$ \emph{demographic parity} if $|\rho_{\text{maj}}^{(\pi)}-\rho_{\text{min}}^{(\pi)}|\le\epsilon$.
\end{definition}
For an individual $(\tilde{s},z)\in S$, $z$ encodes whether they are from the majority ($z=\text{maj}$) or minority ($z=\text{min}$) subpopulation and $\tilde{s}$ encodes their non-sensitive characteristics (e.g., probability of repaying a loan); demographic parity says the cumulative expected individual rewards are equal for the majority and minority subpopulations. Next, we have the following analog of equal opportunity~\citep{hardt2016equality}:
\begin{definition}
\rm
Let $\epsilon\in\mathbb{R}_+$, let $M$ be an MDP with states $S=Z\times Y\times\tilde{S}$, where $Z=\{\text{maj},\text{min}\}$ and $Y=\{\text{qual},\text{unqual}\}$, and let $\rho\in\mathbb{R}^{|S|\times|A|}$ be the individual rewards. For each $z\in Z$, let
\begin{align*}
\rho_z^{(\pi)}&=\mathbb{E}_{(s,a)\sim\Lambda_z^{(\pi)}}[\rho_{s,a}] \\
\Lambda_z^{(\pi)}&=\Lambda^{(\pi)}\mid\exists\tilde{s}\in\tilde{S}~.~s_0=(z,\text{qual},\tilde{s}).
\end{align*}
A policy $\pi$ is $\epsilon$ \emph{equal opportunity} if $|\rho_{\text{maj}}^{(\pi)}-\rho_{\text{min}}^{(\pi)}|\le\epsilon$.
\end{definition}
This property is similar to demographic parity, but where $\Lambda_z^{(\pi)}$ is restricted to the \emph{qualified} subpopulation (i.e., $y=\text{qual}$). In other words, this property says that cumulative expected individual rewards are equal on average for qualified members of the majority and minority subpopulations.
\begin{remark}
\rm
In general, our algorithms apply to any fairness constraint that two subpopulations should have equal expected outcomes---i.e., for any $S_{\text{maj}},S_{\text{min}}\subseteq S$, letting $\rho_z^{(\pi)}=\mathbb{E}_{(s,a)\sim\Lambda_z^{(\pi)}}[\rho_{s,a}]$ and $\Lambda_z^{(\pi)}=\Lambda^{(\pi)}\mid\mathbb{I}[s_0\in S_z]$, the constraint $|\rho_{\text{maj}}^{(\pi)}-\rho_{\text{min}}^{(\pi)}|\le\epsilon$. They also extend to one-sided inequalities and to multiple majority and minority subpopulations. They also extend to batch decisions; see Appendix~\ref{sec:appendixdiscussion}.
\end{remark}
We focus on demographic parity when describing our algorithms, but our results are general. Letting $\Pi_{\text{DP},\epsilon}$ be the class of policies satisfying demographic parity, our goal is to compute the optimal policy
\begin{align}
\label{eqn:problem}
\pi_{\text{DP}}^*=\operatorname*{\arg\max}_{\pi\in\Pi_{\text{DP},\epsilon}}R^{(\pi)}.
\end{align}
We primarily focus is on settings where the MDP is known, which includes settings where the decision-maker learns about individuals via their interactions (see example below), but not ones where they learn across individuals. We describe a basic extension to unknown MDPs in Section~\ref{sec:rl}.

\textbf{Example.}
We describe an MDP $M_{\text{loan}}$ that models individuals applying for loans. We assume each individual has a true probability $p$ of repaying their loan. On step $t$, the bank has an estimate of the distribution of $p$ (e.g., a credit score); we assume this distribution is a Beta distribution---i.e., $p_t\sim\text{Beta}(\alpha_t,\beta_t)$. Thus, the states of our MDP $(\alpha_t,\beta_t)$.
\footnote{Technically, our MDP is the belief MDP of the POMDP where the state $p$ is unobserved.}
The actions are to offer ($a=1$) or deny ($a=0$) a loan. If the bank offers a loan, the transitions are
\begin{align*}
(\alpha_{t+1},\beta_{t+1})=
\begin{cases}
(\alpha_t+1,\beta_t)&\text{with probability }p_t \\
(\alpha_t,\beta_t+1)&\text{with probability }1-p_t.
\end{cases}
\end{align*}
If the bank denies the loan, the transitions are $(\alpha_{t+1},\beta_{t+1})=(\alpha_t,\beta_t)$. However, since we are interested in detrimental effects of the bank's decisions, we assume this decision reduces the applicant's ability to pay for future loans---i.e., $(\alpha_{t+1},\beta_{t+1})=(\alpha_t,\beta_t+\tau)$, where $\tau\in\mathbb{R}_+$ is a hyperparameter. We assume the initial state distribution is $z\sim\text{Bernoulli}(p_Z)$ and $(\alpha,\beta)\sim p_0(\alpha,\beta\mid z)$ for some $p_Z\in[0,1]$ and some distribution $p_0$---i.e., the initial distribution over the parameters $\alpha,\beta$ depends on the whether the applicant is from the majority or minority subpopulation. Note that $p_0$ can additionally be conditioned individual covariates if available. Now, the bank's rewards are
\begin{align}
\label{eqn:examplereward}
\mathbb{E}_{\delta}[\delta I-(1-\delta)P]-\lambda\sqrt{\text{Var}_{\delta}[\delta I-(1-\delta)P]},
\end{align}
where $P$ is the principal (without loss of generality, we let $P=1$), $I$ is interest, $\delta$ indicates whether the loan is repaid, and $\lambda\in\mathbb{R}_+$. The first term is expected profit and the second term is to risk aversion. We assume the goal of the bank is to maximize (\ref{eqn:examplereward}).

The individual rewards are $\mathbb{I}[a=1]$, where $\mathbb{I}$ is the indicator function---i.e., the reward is 1 if the loan is offered and 0 if it is denied. Then, demographic parity says that loans should be given to majority and minority members with equal frequency (within an $\epsilon$ tolerance), and equal opportunity says that loans should be given to qualified majority and minority members at equal rates (we assume an applicant is qualified if their true probability of repaying satisfies $p\ge p_0$ for some $p_0\in[0,1]$).

\begin{algorithm*}[tb]
\caption{Algorithm for finite state, separable MDPs.}
\label{alg:modelbased}
\begin{algorithmic}
\Procedure{LearnFairPolicy}{Separable MDP $M$}
\State Compute the solution $\lambda^*$ to the linear program
\begin{align*}
&\operatorname*{\arg\max}_{\lambda\in\mathbb{R}^{|S|\times|A|}}~(1-\gamma)^{-1}\sum_{s\in S}\sum_{a\in A}\lambda_{s,a}R_{s,a} \\
&\operatorname*{subj. to}~\sum_{a\in A}\lambda_{s',a}=(1-\gamma)D_{s'}+\gamma\sum_{s\in S}\sum_{a\in A}\lambda_{s,a}P_{s,a,s'}\hspace{0.2in}(\forall s'\in S) \\
&\hspace{0.523in}\bigg|~p_{\text{maj}}^{-1}\sum_{\tilde{s}\in\tilde{S}}\sum_{a\in A}\lambda_{(\text{maj},\tilde{s}),a}\rho_{(\text{maj},\tilde{s}),a}-p_{\text{min}}^{-1}\sum_{\tilde{s}\in\tilde{S}}\sum_{a\in A}\lambda_{(\text{min},\tilde{s}),a}\rho_{(\text{min},\tilde{s}),a}~\bigg|\le\epsilon
\end{align*}
\State \textbf{return} $\pi^*$, where $\pi^*_{s,a}=\dfrac{\lambda^*_{s,a}}{\sum_{a'\in A}\lambda^*_{s,a'}}$
\EndProcedure
\end{algorithmic}
\end{algorithm*}

\textbf{Separable MDPs.}
We focus primarily on MDPs where the fairness attribute is constant.
\begin{definition}
\label{def:separable}
\rm
An MDP with states $S=Z\times\tilde{S}$ is \emph{separable} if the transitions satisfy $P_{(z,\tilde{s}),a,(z',\tilde{s}')}=\delta_{z,z'}\tilde{P}_{\tilde{s},a,\tilde{s}'}$, where $\delta_{z,z'}=\mathbb{I}[z=z']$ and $\tilde{P}\in\mathbb{R}^{|\tilde{S}|\times|A|\times|\tilde{S}|}$ is a transition matrix.
\end{definition}
That is, the transitions do not affect $z$, so the sensitive attribute $z\in Z$ does not change over time. This property is satisfied by many sensitive attributes (e.g., race and gender). Fairnes properties may not make sense when the sensitive attribute can change.

\textbf{Existence and determinism.}
Unconstrained MDPs always have a deterministic optimal policy~\citep{sutton2018reinforcement}; however, with a fairness constraint, this result may not hold:
\begin{theorem}
\label{thm:nondeterministic}
There exists $\epsilon>0$ and an MDP $M$ such that $\Pi_{\text{DP},\epsilon}=\varnothing$. There exists $\epsilon>0$ and an MDP $M$ such that $\pi^*$ in (\ref{eqn:problem}) is not deterministic.
\end{theorem}
We give a proof in Appendix~\ref{sec:nondeterministicproof}. For the following special case, we can prove existence of fair policies:
\begin{definition}
\rm
We say $\rho$ is \emph{state-independent} if for some $\tilde{\rho}\in\mathbb{R}^{|A|}$, we have $\rho_{s,a}=\tilde{\rho}_a$ for all $s\in S$.
\end{definition}
Intuitively, this property captures settings where the decision-maker uses the state to choose actions (e.g., ability to repay), but the outcomes for the individuals only depend on whether the preferred action is taken (e.g., a loan offer). Our example $M_{\text{loan}}$ has state-independent individual rewards.
\begin{theorem}
If the individual rewards are state-independent, then (\ref{eqn:problem}) has a solution.
\end{theorem}
\begin{proof}
Any policy $\pi$ such that $\pi_{s,a}=\tilde{\pi}_a$ for all $s\in S$ and some $\tilde{\pi}\in\mathbb{R}^{|A|}$, satisfies $\pi\in\Pi_{\text{DP}}$.
\end{proof}

\textbf{Comparison to supervised learning.}
Our fairness definitions are natural generalizations of their counterparts for supervised learning. For example, in the supervised learning setting, demographic parity says that majority and minority members should, on average, be given positive outcomes at equal rates. Our extension to MDPs says that this property should hold on average across time---more precisely, averaged over $t\sim\text{Geometric}(\gamma)$, where $\gamma$ is the discount factor.

Conversely, our constraint reduces to the supervised learning constraint setting when the state distribution is constant over time---i.e., $D^{(\pi,t)}=D$ is independent of $t$ and $\pi$. To see this claim, note that a constant state distribution implies that $D^{(\pi)}=D$, so the state-action distribution is simply $\Lambda_{s,a}^{(\pi)}=D_s\pi_{s,a}$, and our MDP demographic parity constraint reduces to
\begin{align*}
\big|\mathbb{E}_{s\sim D_{\text{maj}},a\sim\pi_s}[\rho_{s,a}]-\mathbb{E}_{s\sim D_{\text{min}},a\sim\pi_s}[\rho_{s,a}]\big|\le\epsilon.
\end{align*}
In other words, the policy $\pi$ should equalize the expected individual rewards for the majority and minority subpopulations on the initial (constant) state distribution. Finally, assuming the individual rewards are $\rho_{s,a}=1$ for a positive outcome and $\rho_{s,a}=0$ otherwise, then our constraint is equivalently
\begin{align*}
\big|\mathbb{P}_{s\sim D_{\text{maj}},a\sim\pi_s}[\hat{y}=1]-\mathbb{P}_{s\sim D_{\text{min}},a\sim\pi_s}[\hat{y}=1]\big|\le\epsilon,
\end{align*}
where $\hat{y}=\rho_{s,a}$ is the outcome, which is demographic parity for supervised learning~\citep{hardt2016equality}.

Additionally, we introduce individual rewards $\rho$, which may differ from the decision maker rewards $R$. This distinction also appears in the supervised learning setting if the loss function for the decision maker (used in the learning objective) differs from the loss function of the individual (used in the fairness constraint). For example, $R$ may differ from $\rho$ if the decision maker is risk-averse; then, the decision maker may offer too few loans to minorities if there is less historical information available for minorities. We believe this distinction is particularly important to explicitly model in the MDP setting, since dynamical effects can magnify the negative consequences of unfair decision making.

\textbf{Importance of dynamics.}
Dynamics are important when current decisions do not immediately cause unfairness, but can affect the state distribution in a way that leads to unfair outcomes in the future. In our loan applicant example, there are two effects of decisions on the state distribution. First, there is a direct effect---e.g., denying loans can cause adverse outcomes on an applicant's financial situation. In $M_{\text{loan}}$, this effect is captured by the update $\beta_{t+1}=\beta_t+\tau$ when $a=0$---i.e., the applicant's probability of repaying future loans decreases when they are denied a loan.

The second effect is indirect, and is related to the selective labels problem in sequential decision making~\citep{lakkaraju2017selective,bechavod2019equal}. In particular, the bank only observes outcomes if they offer the applicant a loan. A key concern is that less historical information is available for minorities, leading to higher variance estimates of their ability to repay a loan. Thus, a risk-averse decision maker might conservatively deny loans to minorities, even if their expected rate of repaying loans is equal to that of majority members. The equal opportunity constraint forces the decision maker to give exploratory loans to avoid unfairly denying loans to an applicant for whom little data is available.

\section{Algorithm for Finite-State MDPs}
\label{sec:modelbased}

We describe an algorithm for solving (\ref{eqn:problem}), which has strong theoretical guarantees (i.e., it solves (\ref{eqn:problem}) exactly in polynomial time). On the other hand, it makes strong assumptions---i.e., that $M$ has finite state and action spaces. In Section~\ref{sec:general}, we describe a model-free algorithm that applies very generally (e.g., to continuous state and action spaces, or even non-separable MDPs), but lacks performance guarantees.

Our approach is based on the dual of the standard LP formulation of value iteration~\citep{altman1999constrained,sutton2018reinforcement}. In particular, the objective and first set of constraints of the LP in Algorithm~\ref{alg:modelbased} form the dual. The last set of constraints in the LP in Algorithm~\ref{alg:modelbased} encodes demographic parity. These constraints exploit the separable structure of the underlying MDP. In particular, the component $z$ of an initial state $s=(z,\tilde{s})$ does not change over time, so the value of $z$ for $s$ equals the value of $z$ for the initial state $s_0\sim D$. Thus, randomly sampling a state $s\sim D^{(\pi)}_z$ is equivalent to randomly sampling
\begin{align*}
s\sim D^{(\pi)}\mid\exists\tilde{s}\in\tilde{S}~.~s=(z,\tilde{s}).
\end{align*}
Expanding the conditional probability, the probability of sampling $s\sim D^{(\pi)}_z$ is
\begin{align*}
\frac{D^{(\pi)}_s\mathbb{I}[\exists\tilde{s}\in\tilde{S}~.~s=(z,\tilde{s})]}{p_z},\hspace{0.1in}\text{where}~p_z=\sum_{\tilde{s}\in\tilde{S}}D_{(z,\tilde{s})}.
\end{align*}
It follows that
\begin{align}
\label{eqn:separablepz}
\rho_z^{(\pi)}=\mathbb{E}_{(s,a)\sim\Lambda_z^{(\pi)}}[\rho_{s,a}]=p_z^{-1}\sum_{\tilde{s}\in\tilde{S}}\sum_{a\in A}\lambda_{(z,\tilde{s}),a}\rho_{s,a}.
\end{align}
The last set of constraints in the LP in Algorithm~\ref{alg:modelbased} uses (\ref{eqn:separablepz}) to encode demographic parity.
\begin{theorem}
\label{thm:modelbased}
Algorithm~\ref{alg:modelbased} returns a solution $\pi^*$ to (\ref{eqn:problem}) if and only if (\ref{eqn:problem}) is satisfiable.
\end{theorem}
We give a proof in Appendix~\ref{sec:modelbasedproof}. Note that Algorithm~\ref{alg:modelbased} runs in polynomial time.

\begin{remark}
\rm
We briefly compare our approach to algorithms for solving constrained MDPs. Existing approaches are also based on the dual of the LP for solving MDPs~\citep{altman1999constrained}. Indeed, in the LP we use in Algorithm~\ref{alg:modelbased}, the objective and the first constraint are taken from the dual. The second constraint, which encodes the fairness constraint, is novel---our key insight is that for separable MDPs, the fairness constraint can be expressed as a linear inequality over $\lambda$.
\end{remark}

\section{Algorithm for General MDPs}
\label{sec:general}

Next, we propose a general algorithm for solving (\ref{eqn:problem}). However, in general, the planning problem may be non-convex, so unlike Algorithm~\ref{alg:modelbased}, this algorithm may converge to a local optimum.

Our algorithm relies on the cross-entropy (CE) method~\citep{mannor2003cross,hu2012stochastic}, a heuristic for solving optimization problems. Suppose our policies $\pi_{\theta}\in\Pi$ are parameterized by $\theta\in\Theta$, and let a family $\mathcal{F}$ of probability distributions over $\Theta$ parameterized by $V\subseteq\mathbb{R}^d$. We use $\theta$ and $\pi_{\theta}$ interchangeably, e.g., $R^{(\theta)}=R^{(\pi_{\theta})}$. In the unconstrained setting, CE aims to solve the following optimization problem:
\begin{align}
\label{eqn:ceopt}
v^*=\operatorname*{\arg\max}_{v\in V}\mathbb{E}_v[R^{(\theta)}],
\end{align}
where $\mathbb{E}_v=\mathbb{E}_{\theta\sim f_v}$. In other words, it aims to compute a distribution $f_{v^*}$ that places high probability mass on $\theta$ with high cumulative expected reward $R^{(\theta)}$. Then, it returns a sample $\theta\sim f_{v^*}$.
To solve (\ref{eqn:ceopt}), CE starts with initial parameters $v_0\in V$. Then, on each iteration, it updates the current parameters $v_k$ to move ``closer'' to $v^*$. More precisely, the update is
\begin{align}
\label{eqn:ceest}
v_{k+1}&=\operatorname*{\arg\max}_{v\in V}D_{\text{KL}}(g_{k+1}\,\|\,f_v) \\
g_{k+1}(\theta')&=\alpha\frac{R^{(\theta')}\mathbb{I}[R^{(\theta')}\ge\gamma_k]f_{v_k}(\theta')}{\mathbb{E}_{v_k}[R^{(\theta)}\mathbb{I}[R^{(\theta)}\ge\gamma_k]]}+(1-\alpha)f_{v_k}(\theta') \nonumber
\end{align}
where $\gamma_k$ satisfies $\text{Pr}_{v_k}[R^{(\theta)}\ge\gamma_i]=\mu$. Here, $\alpha,\mu\in(0,1)$ are hyperparameters. Intuitively, the first term of $g_i$ upweights $\theta'$ with large values of $R^{(\theta')}$ compared to $f_{v_k}$, both by directly weighting the probability of $\theta'$ by $R^{(\theta')}$, and furthermore by placing zero probability mass on the bottom $1-\mu$ fraction of the $\theta'$. The second term of $g_k$ is a ``smoothing'' term that makes the update incremental.

\begin{algorithm}[tb]
\caption{Algorithm for general MDPs.}
\label{alg:modelfree}
\begin{algorithmic}[1]
\Procedure{GeneralLearnFairPolicy}{MDP $M$, Iters $r$, Parameter samples $n$, Top $n'$, Rollout samples $m$, Smoothing $\alpha$, Tolerance $\sigma$}
\State $\hat{\eta}\gets\vec{0}$
\For{$k\in[1,...,r]$}
\State Sample $\theta^{(1)},...,\theta^{(n)}\sim f_{m^{-1}(\hat{\eta})}$ \label{step:sample}
\For{$i\in[1,...,n]$}
\vspace{0.05in}
\State $\hat{R}^{(\theta^{(i)})}\xleftarrow{\sim m,T}R^{(\theta^{(i)})}$ \label{step:estimateg}
\State $\hat{\epsilon}^{(\theta^{(i)})}\xleftarrow{\sim m,T}|\rho_{\text{maj}}^{(\theta^{(i)})}-\rho_{\text{min}}^{(\theta^{(i)})}|$ \label{step:estimateh}
\EndFor
\State Sort $\{\theta^{(i)}\}_{i=1}^n$ in increasing $\hat{\epsilon}^{(\theta^{(i)})}$
\State $i'\gets$~Largest $i$ such that $\hat{\epsilon}^{(\theta^{(i)})}\le(1-\sigma)\epsilon$
\If{$n'\le i'$} \label{step:check_gamma}
\State Sort $\{\theta^{(i)}\}_{i=1}^{i'}$ in decreasing $\hat{R}^{(\theta^{(i)})}$ \label{step:quantile_end}
\EndIf
\State $\hat{\eta}\gets\alpha\cdot\frac{\frac{1}{n}\sum_{i=1}^{n'}\hat{R}^{(\theta^{(i)})}\Gamma(\theta^{(i)})}{\frac{1}{n}\sum_{i=1}^{n'}\hat{R}^{(\theta^{(i)})}}+(1-\alpha)\cdot\hat{\eta}$ \label{step:update_eta}
\EndFor
\If{$\hat{\epsilon}^{(\hat{\theta})}\le\tilde{\epsilon}$, where $\hat{\theta}\sim f_{m^{-1}(\hat{\eta})}$}
\State \textbf{return} $\pi_{\hat{\theta}}$
\Else
\State \textbf{return} $\varnothing$
\EndIf
\EndProcedure
\end{algorithmic}
\end{algorithm}

To enable efficient optimization of (\ref{eqn:ceest}), we assume that $\mathcal{F}$ is a (natural) exponential family.
\begin{definition}
\rm
A family $\mathcal{F}$ of distributions over $\Theta\subseteq\mathbb{R}^d$ is an \emph{exponential family} if, for a continuous $\Gamma:\Theta\to\mathbb{R}^d$, $f_v(\theta)=e^{v^{\top}\Gamma(\theta)}/Z(\theta)$, where $Z(\theta)=\int e^{v^{\top}\Gamma(\theta)}d\theta$.
\end{definition}
We use the standard choice that $\mathcal{F}$ is the space of Gaussians. If $\mathcal{F}$ is an exponential family, then
\begin{align}
\label{eqn:ceefupdate}
v_{k+1}&=m^{-1}(\eta_{k+1}) \\
\eta_{k+1}&=\alpha\frac{\mathbb{E}_{v_k}[R^{(\theta)}\mathbb{I}[R^{(\theta)}\ge\gamma_k]\Gamma(\theta)]}{\mathbb{E}_{v_k}[R^{(\theta)}\mathbb{I}[R^{(\theta)}\ge\gamma_k]]}+(1-\alpha)\eta_k \nonumber
\end{align}
where $m(v)=\mathbb{E}_v[\Gamma(\theta)]$ is the moment map~\citep{hu2012stochastic}. The CE algorithm approximates (\ref{eqn:ceefupdate}) by sampling rollouts $\zeta=((s_0,a_0),...,(s_{T-1},a_{T-1}))$ according to $\pi_{\theta}$. Then, it computes the estimate $R^{(\theta)}\approx\hat{R}^{(\theta)}=\frac{1}{m}\sum_{i=1}^m\hat{R}(\zeta^{(i)})$, where $\zeta^{(1)},...,\zeta^{(m)}$ are $m$ sampled rollouts and $\hat{R}(\zeta)=\sum_{t=0}^{T-1}\gamma^tR_{s_t,a_t}$.

To estimate $\eta_{k+1}$, it takes $n$ samples $\theta^{(1)},...,\theta^{(n)}\sim f_v$, and computes $\hat{R}^{(\theta^{(i)})}$ for each $i$. Then, it ranks $\theta^{(i)}$ in decreasing order of $\hat{R}^{(\theta^{(i)})}$, and discards all but the top $n'=\lceil n\mu\rceil$. It estimates the numerator in $\eta_{k+1}$ as
\begin{align*}
\mathbb{E}_{v_k}[R^{(\theta)}\mathbb{I}[R^{(\theta)}\ge\gamma_k]\Gamma(\theta)]\approx\frac{1}{n}\sum_{i=1}^{n'}\hat{R}^{(\theta^{(i)})}\Gamma(\theta^{(i)}).
\end{align*}
The denominator in $\eta_{k+1}$ is estimated similarly.

Algorithm~\ref{alg:modelfree} computes this estimate of the update (\ref{eqn:ceefupdate}) assuming the condition on Line 16 is satisfied (as we discuss below, the check is needed to enforce the constraint that $\pi\in\Pi_{\text{DP},\epsilon}$. Line 6 of Algorithm~\ref{alg:modelfree} computes the estimates $\hat{R}^{(\theta^{(i)})}$ for samples $\theta^{(i)}\sim f_{v_k}$ for $i\in[n]$, and Line 14 estimates $\eta_{k+1}$. On Line 6 \& 7, the notation $\xleftarrow{\sim m,T}$ means to estimate a quantity using $m$ sampled rollouts $\zeta^{(1)},...,\zeta^{(m)}$ each of length $T$.

Finally, we adapt constrained cross-entropy (CCE), which extends CE to handle constraints~\citep{wen2018constrained}, to handle fairness constraints. Intuitively, CCE prioritizes policies where the constraint that $\pi\in\Pi_{\text{DP},\epsilon}$ is closer to holding, unless the constraint holds, in which case CCE prioritizes policies with higher cumulative expected reward. In particular, Algorithm~\ref{alg:modelfree} imposes this constraint by checking if $\hat{\theta}$ satisfies the constraint $\hat{\epsilon}^{(\hat{\theta})}\le\epsilon$ in Line 16, where $\hat{\epsilon}^{(\hat{\theta})}$ is estimated from samples. Note that $\tilde{\epsilon}$ is used in place of $\epsilon$ to enforce the constraint even though $\hat{\epsilon}^{(\hat{\theta})}$ is inexact. The reason is that CCE relies on estimates $\hat{\epsilon}^{(\hat{\theta})}$ of $\epsilon^{(\hat{\theta})}$. These estimates are inexact since (i) they are estimated from samples, and (ii) they are estimated based on a finite time horizon (whereas $\epsilon^{(\hat{\theta})}$ is defined for an infinite horizon). To account for this error, we use $(1-\sigma)\epsilon$ (where $\sigma\in(0,1)$) in place of $\epsilon$ when checking the constraint on Line 16 of Algorithm~\ref{alg:modelfree}.

We provide the following for Algorithm~\ref{alg:modelfree} (see Appendix~\ref{sec:modelfreeproof} for a proof).
\begin{theorem}
\label{thm:modelfree}
Assume that $\rho_{\text{max}}$ is an upper bound on $\rho$ (i.e., $\|\rho\|_{\infty}=\rho_{\text{max}}$ for all $z\in Z$). Let $\delta\in\mathbb{R}_+$ and $\sigma\in(0,1/2]$ be given, and suppose that
\begin{align*}
m&\ge\frac{32\rho_{\text{max}}(1-\gamma)\log(4/\delta)}{\sigma\epsilon^2}\quad
T\ge\log\frac{4\rho_{\text{max}}}{\sigma^2\epsilon(1-\gamma)}.
\end{align*}
Then, with probability at least $1-\delta$, we have $\pi_{\hat{\theta}}\in\Pi_{\text{DP},\epsilon}$, where $\pi_{\hat{\theta}}$ is returned by Algorithm~\ref{alg:modelfree}.
\end{theorem}

\begin{figure*}[t]
\centering
\begin{tabular}{cccc}
\includegraphics[width=0.22\textwidth]{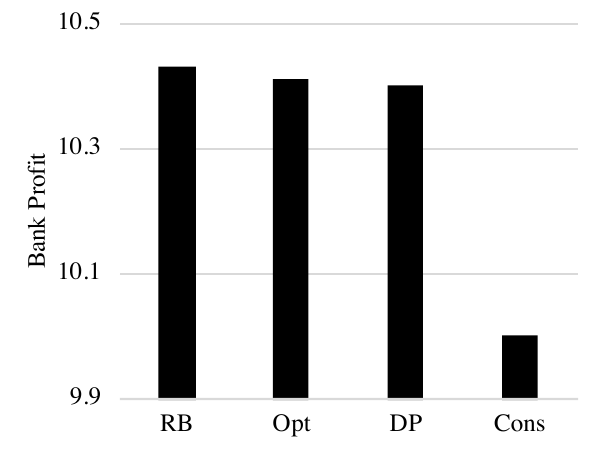} &
\includegraphics[width=0.22\textwidth]{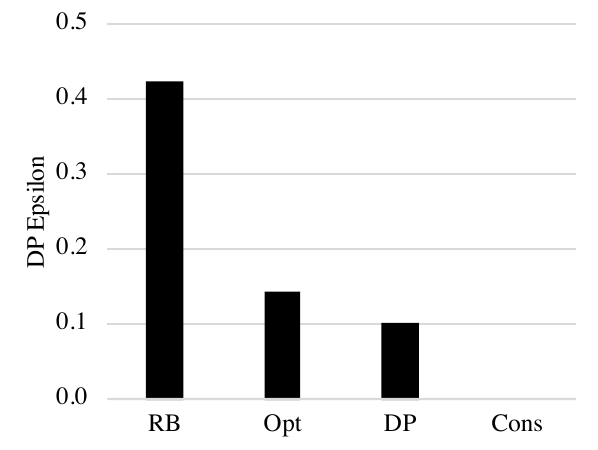} &
\includegraphics[width=0.22\textwidth]{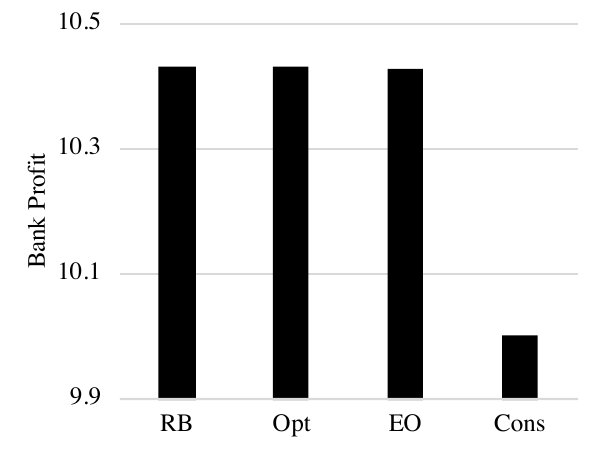} &
\includegraphics[width=0.22\textwidth]{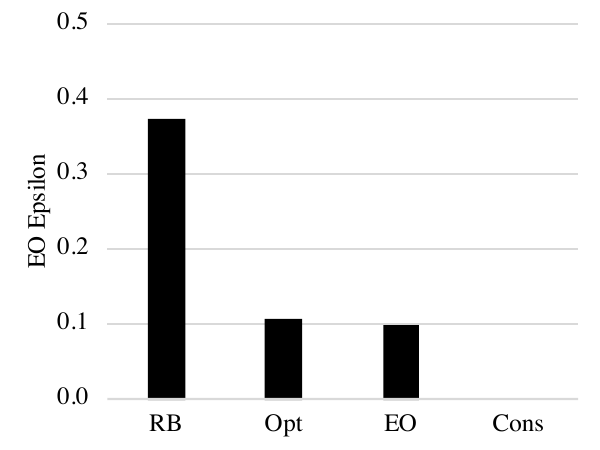} \\
(a) & (b) & (c) & (d)
\end{tabular}
\caption{Demographic parity (a) objective value, (b) constraint value, and equal opportunity (c) objective value, (d) constraint value, for race-blind (RB), demographic parity (DP) or equal opportunity (EO), optimistic (Opt), and conservative (Cons).}
\label{fig:exp}
\end{figure*}

\section{Reinforcement Learning}
\label{sec:rl}

We discuss extensions to the setting where the MDP is initially unknown, and the goal is to ensure fairness while learning these quantities. We propose an approach to fairness when the transitions $P$ are unknown but the initial state distribution $D$ is known; reducing to the case of unknown $D$ is standard (i.e., add a deterministic initial state $s_0$ and transition to an initial state according to $D$). Our goal is to ensure that with high probability, fairness holds for all time including during learning. We consider the episodic case where the system is reset after a fixed number of steps $T$, and take $\gamma=1$. That is, a finite sequence of interactions is performed repeatedly---e.g., each new loan applicant is a new episode. We assume there are a fixed total number of episodes $N$, and the goal is to perform well on average; the doubling trick can be used to generalize to unknown or unbounded $N$ (see p. 99 of \cite{lattimore2018bandit}).

A key challenge is how to design a fair policy we can use when the dynamics are unknown. Thus, we focus on the setting of state-independent individual rewards $\rho$, where we can ensure such a policy exists. In particular, we take $\pi_0$ to choose actions uniformly randomly---i.e., $\pi_0(s,a)=1/|A|$ for all $s\in S$ and $a\in A$. Then, we are guaranteed that $\pi_0$ is fair. Furthermore, we are guaranteed that $\pi_0$ explores all states (assuming without loss of generality that we prune unreachable states)---i.e., letting $D^{(\pi)}=\frac{1}{T}\sum_{t=0}^{T-1}D^{(\pi,t)}$ and $\Lambda^{(\pi)}_{s,a}=D^{(\pi)}_s\pi_{s,a}$, where $D^{(\pi,t)}$ is defined as before, then there exists $\lambda_0\in\mathbb{R}_+$ such that
\begin{align*}
\Lambda^{(\pi_0)}_{s,a}\ge\lambda_0>0\hspace{0.2in}(\forall s\in S,~a\in A)
\end{align*}

We use explore-then-commit~\citep{lattimore2018bandit}. First, we explore using the conservative policy $\pi_0$ for $N_0$ episodes. Then, we estimate $P$ using the observed state-action-state tuples $(s,a,s')$ (i.e., transition to $s'$ upon taking action $a$ in state $s$):
\begin{align*}
\hat{P}_{s,a,s'}=\frac{\#\text{ observed tuples }(s,a,s')}{\#\text{ observed tuples }(s,a,s'')\text{ for some }s''\in S}.
\end{align*}
Finally, for the remaining $N-N_0$, it uses the optimal policy $\hat{\pi}$ computed as if $\hat{P}$ is the true transition matrix.

We prove a bound on the \emph{regret}
\begin{align*}
\mathcal{R}(N)=\mathbb{E}\left[\sum_{n=1}^NR^{(\pi^*)}-R^{(\pi_n)}\right],
\end{align*}
where the expectation is taken over the randomness of the observed tuples $(s,a,s')$, $\pi^*$ is the optimal policy for known $P$ that satisfies $\pi^*\in\Pi_{\text{DP},\epsilon/4}$, and
\begin{align*}
\pi_n&=\begin{cases}\pi_0&\text{if}~n\le N_0\\\hat{\pi}&\text{otherwise}\end{cases} \\
N_0&=\frac{128T^4\cdot|S|^2\cdot R_{\text{max}}^2\cdot\log(2|S|^2|A|/\delta)}{\lambda_0^2\tilde{\epsilon}^2}.
\end{align*}
is the policy our algorithm uses on episode $n$. We show that $\hat\pi$ is fair, and that given $\delta\in\mathbb{R}_+$, $\pi_n\in\Pi_{\text{DP},\epsilon}$ for every $n\in[N]$ with probability at least $1-\delta$.
\begin{theorem}
\label{thm:unknowntrans}
Let $\epsilon,\delta\in\mathbb{R}_+$ be given. Assume that $R_{\text{max}}$ is an upper bound on $R$ (i.e., $\|R\|_{\infty}=R_{\text{max}}$) and on $\rho$. Let $\hat{M}=(S,A,D,\hat{P},R,T)$, and $\hat{\pi}$ be the optimal policy for $\hat{M}$ in $\hat{\Pi}_{\text{DP},\epsilon/2}$ (i.e., the set of policies satisfying demographic parity for $\hat{M}$). Let $M=(S,A,D,P,R,T)$, and $\pi^*$ be optimal for $M$ in $\Pi_{\text{DP},\epsilon/4}$. Then, $\hat{\pi}\in\Pi_{\text{DP},\epsilon}$, and $\mathcal{R}(N)=O((N^{2/3}+1/\epsilon^2)\log(1/\delta))$ with probability at least $1-\delta$.
\end{theorem}
We give a proof in Appendix~\ref{sec:thmunknowntransproof}. Note that there is a gap between the fairness constraint of $\pi^*$ (which is in $\Pi_{\text{DP},\epsilon/4}$) and that of $\hat{\pi}$ (which is only in $\Pi_{\text{DP},\epsilon}$)---i.e., we can only guarantee performance compared to a policy that satisfies a stricter level of fairness.

\section{Experiments}
\label{sec:exp}

We run simulations using our loan example from Section~\ref{sec:problem}. We estimated parameters based on FICO score data~\citep{hardt2016equality}. We consider Whites to be majorities, and Blacks, Hispanics, and Asians to be minorities. For the initial distribution $p_0$, we first fit parameters the parameters of the prior $\text{Beta}(\alpha_z,\beta_z)$ based on the data. Then, we take a fixed number of steps $T_z$ using action $a=1$ (i.e., offer loan) to force exploration. We choose $T_{\text{maj}}>T_{\text{min}}$ to capture the idea that less data is available for minorities. We also estimate the probability $p_Z$ of being a minority from the data. Similar to~\citep{hardt2016equality}, we choose $I$ so the bank makes a profit on the average applicant. We manually choose $\lambda$, $\tau$, $T_{\text{maj}}$, and $T_{\text{min}}$ based on intuition; see Appendix~\ref{sec:expappendix} for the values we chose. We focus on evaluation of Algorithm~\ref{alg:modelfree}, and give additional experimental results in Appendix~\ref{sec:expappendix}.

\textbf{Baselines that ignore dynamics.}
To demonstrate the importance of accounting for dynamics, we compare to two baselines that ignore dynamics when constraining fairness. The first optimistically pretends that actions do not affect the state distribution---i.e., $D^{(\pi,t)}$ does not change over time. In this case, for all $t>0$, we have $D^{(\pi,t)}=D$, so $D^{(\pi)}=D$ for any $\pi$. Thus, we can let
\begin{align}
\label{eqn:optimisticopt}
\pi^*=&\operatorname*{\arg\max}_{\pi\in\Pi,c\in\mathbb{R}}R^{(\pi)} \\
&\operatorname*{subj. to}~\mathbb{E}_{s\sim D_z}\left[\sum_{a\in A}\pi_{s,a}\rho_{s,a}\right]=c\hspace{0.2in}(\forall z\in Z), \nonumber
\end{align}
where $D_z=D\mid\exists\tilde{s}\in\tilde{S}~.~s_0=(z,\tilde{s})$. We can solve (\ref{eqn:optimisticopt}) using a straightforward modification of Algorithm~\ref{alg:modelfree}. This captures the supervised learning setting. Compared to our algorithm, this algorithm may learn a policy that is unfair but achieves higher reward.

The second conservatively assumes $D^{(\pi,t)}$ can change arbitrarily on each step. This baseline learns a fair policy, but it may achieve much lower reward. In this case, we restrict to policies $\pi$ that satisfy
\begin{align}
\label{eqn:conservativeassump}
\mathbb{E}_{s\sim D_{\text{maj}}'}\left[\sum_{a\in A}\pi_{s,a}\rho_{s,a}\right]=\mathbb{E}_{s\sim D_{\text{min}}'}\left[\sum_{a\in A}\pi_{s,a}\rho_{s,a}\right] \\
\hspace{2.0in}(\forall D'\in\Delta^{|S|}), \nonumber
\end{align}
where $D_z'=D'\mid\exists\tilde{s}\in\tilde{S}~.~s=(z,\tilde{s})$, and $\Delta^n$ is the standard $n$-simplex. Note that $D_z'$ is conditioned on $s=(z,\tilde{s})$ (i.e., the current state has sensitive attribute $z$) instead of $s_0=(z,\tilde{s})$ (i.e., the initial state has sensitive attribute $z$); if $M$ is separable, these two conditions are equivalent. Finally, note that $D'_z$ is undefined if the conditional has zero probability according to $D'$; we implicitly omit such $D'$ from (\ref{eqn:conservativeassump}).

The difficulty with (\ref{eqn:conservativeassump}) is the universal quantification over $D'\in\Delta^{|S|}$. For state-independent individual rewards, the conservative assumption is in fact equivalent to optimizing over state-independent policies---i.e., those of the form $\pi_{s,a}=\tilde{\pi}_a$, where $\tilde{\pi}\in\mathbb{R}^{|A|}$. Thus, we can apply a modified version of Algorithm~\ref{alg:modelfree} where we only learn state-independent policies.

\textbf{Results for Algorithm~\ref{alg:modelfree}.}
We ran Algorithm~\ref{alg:modelfree} to learn fair policies for both the demographic parity and equal opportunity constraints, using $\epsilon=0.1$. For each constraint, we also use our optimistic and conservative baselines. We also consider a race-blind algorithm that is unconstrained but where $\pi$ ignores the sensitive attribute $z\in Z$. The optimal policy is race-blind---the state is a sufficient statistic, so it captures all information needed to determine whether to offer a loan.

For demographic parity, Figure~\ref{fig:exp} (a) shows the reward achieved for the bank, and (b) shows the value of the fairness constraint---i.e., the smallest value of $\epsilon$ for which $\pi\in\Pi_{\text{DP},\epsilon}$. As expected, race-blind achieves the highest reward (10.43), followed by the optimistic algorithm (10.41), and then Algorithm~\ref{alg:modelfree} (10.40). Finally, the conservative algorithm performs substantially worse than the others (10.00). However, race-blind achieves a very poor constraint value (0.42), as does the optimistic algorithm (0.14), which performs performs 43\% worse than Algorithm~\ref{alg:modelfree} (0.10). The conservative algorithm achieves constraint value 0. For equal opportunity, Figure~\ref{fig:exp} (c) shows the bank reward, and (d) shows the value of the constraint. The bank's rewards are essentially the same for the race-blind algorithm, optimistic algorithm, and Algorithm~\ref{alg:modelfree} (10.43), but is substantially worse for the conservative algorithm (10.00). As with demographic parity, the constraint value for race-blind (0.37) is substantially worse than the others, but in this case optimistic (0.11) is fairly close to Algorithm~\ref{alg:modelfree} (0.10). The conservative algorithm achieves constraint value 0.

\textbf{Discussion.}
Our results show that imposing demographic parity slightly reduces the bank's reward, but substantially increases fairness compared to the race-blind and optimistic algorithms. The latter models supervised learning---thus, our results show the importance of accounting for dynamics when ensuring fairness. We find similar (but weaker) trends for equal opportunity. Like prior work~\citep{hardt2016equality}, we find that demographic parity reduces the bank's rewards more than equal opportunity.

Unlike the static case~\citep{hardt2016equality}, our model has dynamic parameters. Time series data would be needed to estimate them; instead, we choose them manually. Also, \citep{hardt2016equality} uses the empirical CDF of the distribution over repayment probabilities $p_0$, whereas we assumed $p_0$ is a Beta distribution. Our goal is to understand the consequences of ignoring dynamics, not to study a real-world scenario.

\section{Conclusion}

We have proposed algorithms to learn fair policies that account for the dynamical effects, and have demonstrated the importance of accounting for these effects. There is much room for future work. One important direction is extending our results for the case where the initial MDP is unknown beyond explore-then-commit to obtain better regret guarantees. Another direction is theoretically analyzing the cost of fairness---e.g., what is the cost to the bank for imposing a fairness constraint, and how they can mitigate this cost by improving predictive power. Finally, reinforcement learning problems in practice are often offline---i.e., the goal is to learn from historical data and the algorithm does not have the opportunity to explore. Studying fairness in this context is an important problem.

\subsubsection*{Acknowledgements}

We thank the anonymous reviewers for their insightful comments. This research was partially supported by NSF awards 1910769 and 1652113.

%% file: appendix.tex
\section{Batch Decisions}
\label{sec:appendixdiscussion}

We describe how our algorithm can be extended to the setting where the decision-maker makes decisions about batches of individuals jointly rather than one individual at a time. For instance, a bank might decide on a portfolio of loans to target at once rather than decide independently for each individual. The challenge is ensuring decisions are fair not just across different batches of individuals, but also across individuals within a batch, since these decisions may be correlated.

In this setting, the state space becomes $S'=S^k=(Z\times\tilde{S})^k$, where there are $k$ individuals and $S$ is the state space of individual $i$. The action space is $A'=A^k=\{0,1\}^k$---i.e., a binary decision for each individual. Then, the state-action distribution is $\lambda\in\mathbb{R}^{|S'|\times|A'|}$, with a component $\lambda_{s,a}=\lambda_{(s_1,...,s_k),(a_1,...,a_k)}$ for each state $s=(s_1,...,s_k)\in S'$ and action $a=(a_1,...,a_k)\in A'$. The policy $\pi$ can simultaneously make decisions for all $k$ individuals. We let the individual rewards for individual $i$ be $\rho^i\in\mathbb{R}^{|S'|\times|A'|}$. Then, the natural generalization of our fairness constraint is that decisions should be fair on average across both the initial state distribution and across individuals in a single batch. For instance, demographic parity says that
\begin{align*}
\left|\frac{1}{k}\sum_{i=1}^k\rho^{(\pi)}_{i,\text{maj}}-\frac{1}{k}\sum_{i=1}^k\rho^{(\pi)}_{i,\text{min}}\right|\le\epsilon,
\end{align*}
where
\begin{align*}
\begin{array}{l}
\rho^{(\pi)}_{i,z}=\mathbb{E}_{(s,a)\sim\Lambda_{i,z}^{(\pi)}}[\rho_{s,a}^i] \\
\Lambda_{i,z}^{(\pi)}=\Lambda^{(\pi)}\mid\exists\tilde{s}_i\in\tilde{S}~.~s_0=(...,(z,\tilde{s}_i),...),
\end{array}
\end{align*}
This constraint can be encoded in our linear program in Algorithm 1 by replacing the second constraint with the following (the objective and first constraint remain the same, except with $S$ replaced by $S'$ and $A$ replaced by $A'$):
\begin{align*}
|X_{\text{maj}}-X_{\text{min}}|\le\epsilon
\end{align*}
where
\begin{align*}
&X_z= \\
&\frac{1}{k}\sum_{i=1}^kp_{i,z}^{-1}\sum_{s_j\in S_j:j\neq i}\sum_{\tilde{s}_i\in\tilde{S}}\sum_{a\in A'}\lambda_{(...,(z,\tilde{s}_i),...),a}\rho_{(...,(z,\tilde{s}_i),...),a},
\end{align*}
and where $p_{i,z}$ is a normalizing constant similar to $p_z$ in the original constraint. Intuitively, this constraint is the same as the original one except that we marginalize over all other individuals (i.e., the sums over $s_j$ for $j\neq i$), and then we average over individuals $i$ as in our fairness constraint. This constraint can similarly be incorporated into Algorithm~\ref{alg:modelfree}. Finally, while this MDP has number of states exponential in the number of individuals $k$, this blowup is inevitable since the policy is allowed to make complex decisions based on the states of all individuals.

\section{Proof of Theorem~\ref{thm:nondeterministic}}
\label{sec:nondeterministicproof}

For the first claim, consider the MDP $M$. The states are $s_0,s_1,s_2,s_3,s_4\in\tilde{S}\times Z$, where:
\begin{align*}
s_0&=(0,\text{maj}) \\
s_1&=(1,\text{maj}) \\
s_2&=(0,\text{min}) \\
s_3&=(1,\text{min}) \\
s_4&=(2,\text{min}).
\end{align*}
The actions are $A=\{0,1\}$. The transitions are
\begin{align*}
P_{s_0,a,s_1}&=1 \\
P_{s_1,a,s_1}&=1 \\
P_{s_2,a,s_3}&=\mathbb{I}[a=0] \\
P_{s_2,a,s_4}&=\mathbb{I}[a=1] \\
P_{s_3,s_3}&=1 \\
P_{s_4,s_4}&=1
\end{align*}
for all $a\in A$. The initial distribution is
\begin{align*}
&D_{s_0}=D_{s_2}=\frac{1}{2} \\
&D_{s_1}=D_{s_3}=D_{s_4}=0.
\end{align*}
The discount factor is $\gamma=\frac{1}{2}$. The individual rewards are
\begin{align*}
\rho_{s_0,a}&=0 \\
\rho_{s_1,a}&=1 \\
\rho_{s_2,a}&=0 \\
\rho_{s_3,a}&=0 \\
\rho_{s_4,a}&=2,
\end{align*}
for all $a\in A$. Let $\pi:S\to A$ be a deterministic policy. It is clear that the only value of $\pi$ that matters is $\pi(s_2)$. Conditioned on $z=\text{maj}$, regardless of $\pi$, the expected cumulative individual reward is
\begin{align*}
\mathbb{E}_{(s,a)\sim\Lambda_{\text{maj}}^{(\pi)}}[\rho_{s,a}]
&=\left(1-\frac{1}{2}\right)\sum_{t=1}^{\infty}\frac{1}{2^t} \\
&=\frac{1}{2}.
\end{align*}
Conditioned on $z=\text{min}$, if $\pi(s_2)=0$, then
\begin{align*}
\mathbb{E}_{(s,a)\sim\Lambda_{\text{min}}^{(\pi)}}[\rho_{s,a}]&=
\begin{cases}
0&\text{if}~\pi(s_2)=0 \\
1&\text{if}~\pi(s_2)=1.
\end{cases}
\end{align*}
Thus, for $\epsilon<\frac{1}{2}$, it is impossible for the demographic parity constraint to be satisfied.

However, consider the stochastic policy
\begin{align*}
\pi_{s_2,0}=\pi_{s_2,1}=\frac{1}{2}.
\end{align*}
Then, 
\begin{align*}
\mathbb{E}_{(s,a)\sim\Lambda_{\text{min}}^{(\pi)}}[\rho_{s,a}]=\frac{1}{2},
\end{align*}
so this policy satisfies the demographic parity constraint.

For the second claim, consider the same MDP, except where
\begin{align*}
\rho_{s_4,a}=0
\end{align*}
for all $a\in A$. Then, it is clear that
\begin{align*}
\mathbb{E}_{(s,a)\sim\Lambda_{\text{min}}^{(\pi)}}[\rho_{s,a}]=0
\end{align*}
regardless of $\pi$. Thus, for $\epsilon<\frac{1}{2}$, the demographic parity constraint cannot be satisfied---i.e., $\Pi_{\text{DP},\epsilon}=\varnothing$. $\qed$

\section{Proof of Theorem~\ref{thm:modelbased}}
\label{sec:modelbasedproof}

Our proof proceeds in three steps. First, we show that any feasible point of the LP in Algorithm~\ref{alg:modelbased} is the state-action distribution $\Lambda^{(\pi)}$ for some policy $\pi\in\Pi_{\text{DP}}$. Second, we show that conversely, for any fair policy $\pi\in\Pi_{\text{DP}}$, the state-action distribution $\Lambda^{(\pi)}$ is a feasible point of the LP. Finally, we combine these two results to prove the theorem.

\paragraph{Step 1.}

Let $\pi\in\Pi_{\text{DP}}$ be any policy satisfying demographic parity. Then, we claim that the state-action distribution $\Lambda^{(\pi)}$ is a feasible point of the LP in Algorithm~\ref{alg:modelbased}.

First, we show that $\Lambda^{(\pi)}$ satisfies the first constraint
\begin{align*}
\sum_{a\in A}\Lambda^{(\pi)}_{s',a}=(1-\gamma)D_{s'}+\gamma\sum_{s\in S}\sum_{a\in A}\Lambda^{(\pi)}_{s,a}P_{s,a,s'}
\end{align*}
for each $s'\in S$.

To this end, note that by induction,
\begin{align*}
D^{(\pi,t)}=(P^{(\pi)})^tD,
\end{align*}
so
\begin{align}
\label{eqn:alg1proof:dformula1}
D^{(\pi)}=(1-\gamma)\left[\sum_{t=0}^{\infty}(\gamma P^{(\pi)})^t\right]D.
\end{align}
Multiplying each side of (\ref{eqn:alg1proof:dformula1}) by $I-\gamma P^{(\pi)}$ (where $I$ is the $|S|\times|S|$ identity matrix), we have
\begin{align*}
&(I-\gamma P^{(\pi)})D^{(\pi)} \\
&=(1-\gamma)\left[\sum_{t=0}^{\infty}(\gamma P^{(\pi)})^t-\sum_{t=1}^{\infty}(\gamma P^{(\pi)})^t\right]D \\
&=(1-\gamma)\cdot D.
\end{align*}
Note that these algebraic manipulations are valid since the eigenvalues of $\gamma P^{(\pi)}$ are bounded in norm by $\gamma<1$, so all sums converge absolutely. Rearranging this equality gives
\begin{align}
\label{eqn:alg1proof:dformula2}
D^{(\pi)}=(1-\gamma)D+\gamma P^{(\pi)}D^{(\pi)}.
\end{align}
It follows that
\begin{align*}
\sum_{a\in A}\Lambda^{(\pi)}_{s',a}=(1-\gamma)D_{s'}+\gamma\sum_{s\in S}\sum_{a\in A}\Lambda^{(\pi)}_{s,a}P_{s,a,s'}
\end{align*}
for each $s'\in S$, where we have used the equalities
\begin{align*}
D^{(\pi)}_{s'}=\sum_{a\in A}\Lambda^{(\pi)}_{s',a}
\end{align*}
and
\begin{align*}
(P^{(\pi)}D^{(\pi)})_{s'}&=\sum_{s\in S}P^{(\pi)}_{s,s'}D^{(\pi)}_s \\
&=\sum_{s\in S}\sum_{a\in A}P_{s,a,s'}\pi_{s,a}D^{(\pi)}_s \\
&=\sum_{s\in S}\sum_{a\in A}P_{s,a,s'}\Lambda^{(\pi)}_{s,a}
\end{align*}
that follow from the definition of $\Lambda^{(\pi)}$. Therefore, $\Lambda^{(\pi)}$ satisfies the first constraint.

Next, we show that $\Lambda^{(\pi)}$ satisfies the second constraint, which says that
\begin{align}
\label{eqn:alg1proof:secondconstraint2}
&\bigg|~p_{\text{maj}}^{-1}\sum_{\tilde{s}\in\tilde{S}}\sum_{a\in A}\lambda_{(\text{maj},\tilde{s}),a}\rho_{(\text{maj},\tilde{s}),a} \\
&\hspace{0.1in}-p_{\text{min}}^{-1}\sum_{\tilde{s}\in\tilde{S}}\sum_{a\in A}\lambda_{(\text{min},\tilde{s}),a}\rho_{(\text{min},\tilde{s}),a}~\bigg|\le\epsilon. \nonumber
\end{align}
In particular, note that
\begin{align*}
D^{(\pi)}_z=D^{(\pi)}\mid\exists\tilde{s}\in\tilde{S}~.~s=(z\tilde{s}),
\end{align*}
since the value of $z$ for $s$ equals the value of $z$ for the initial state $s_0\sim D$. Furthermore, the probability of sampling $s\sim D^{(\pi)}\mid\exists\tilde{s}\in\tilde{S}~.~s=(z,\tilde{s})$ is
\begin{align*}
\frac{D^{(\pi)}_s\mathbb{I}[\exists\tilde{s}\in\tilde{S}~.~s=(z,\tilde{s})]}{p_z}.
\end{align*}
Together with the definition of $\Lambda^{(\pi)}_z$, we have
\begin{align*}
(\Lambda^{(\pi)}_z)_{s,a}
&=(D^{(\pi)}_z)_s\pi_{s,a} \\
&=\frac{D^{(\pi)}_s\mathbb{I}[\exists\tilde{s}\in\tilde{S}~.~s=(z,\tilde{s})]}{p_z}\cdot\pi_{s,a} \\
&=\frac{\Lambda^{(\pi)}_{s,a}\mathbb{I}[\exists\tilde{s}\in\tilde{S}~.~s=(z,\tilde{s})]}{p_z}.
\end{align*}
Therefore, we have
\begin{align}
&\mathbb{E}_{(s,a)\sim\Lambda^{(\pi)}_z}[\rho_{s,a}] \nonumber \\
&=\sum_{s\in S}\sum_{a\in A}\frac{\Lambda^{(\pi)}_{s,a}\mathbb{I}[\exists\tilde{s}\in\tilde{S}~.~s=(z,\tilde{s})]}{p_z}\cdot\rho_{s,a} \nonumber \\
&=p_z^{-1}\sum_{\tilde{s}\in\tilde{S}}\sum_{a\in A}\Lambda^{(\pi)}_{s,a}\rho_{s,a}.
\label{eqn:alg1proof:secondconstraint1}
\end{align}
By assumption, $\pi$ satisfies the demographic parity constraint, which says exactly that (\ref{eqn:alg1proof:secondconstraint1}) satisfies (\ref{eqn:alg1proof:secondconstraint2}). Thus, $\Lambda^{(\pi)}$ satisfies the second constraint.

Therefore, $\Lambda^{(\pi)}$ is a feasible point of the LP, as claimed.

\paragraph{Step 2.}

Let $\lambda\in\mathbb{R}^{|S|\times|A|}$ be a feasible point of the LP in Algorithm~\ref{alg:modelbased}, and let
\begin{align*}
\pi_{s,a}=\frac{\lambda_{s,a}}{\sum_{a'\in A}\lambda_{s,a'}}
\end{align*}
be the corresponding policy returned by Algorithm~\ref{alg:modelbased}. Then, we claim that $\lambda=\Lambda^{(\pi)}$, that $\pi\in\Pi_{\text{DP}}$, and that the value of the objective for $\lambda$ equals $R^{(\pi)}$.

To see the first claim, let $d\in\mathbb{R}^{|S|}$ be defined by
\begin{align*}
d_s=\sum_{a\in A}\lambda_{s,a}.
\end{align*}
We show that $D^{(\pi)}=d$. To this end, note that because $\lambda$ satisfies the first constraint in the LP, we have
\begin{align*}
\sum_{a\in A}\lambda_{s',a}=(1-\gamma)D_{s'}+\gamma\sum_{s\in S}\sum_{a\in A}\lambda_{s,a}P_{s,a,s'}.
\end{align*}
Together with the equality
\begin{align*}
\pi_{s,a}=\frac{\lambda_{s,a}}{d_s},
\end{align*}
we have
\begin{align*}
d_{s'}
&=(1-\gamma)D_{s'}+\gamma\sum_{s\in S}\sum_{a\in A}d_s\pi_{s,a}P_{s,a,s'} \\
&=(1-\gamma)D_{s'}+\gamma(P^{(\pi)}d)_{s'}.
\end{align*}
Thus,
\begin{align}
\label{eqn:alg1proof:dformula3}
d=(1-\gamma)D+\gamma P^{(\pi)}d.
\end{align}
We note that $I-\gamma P^{(\pi)}$ is invertible---in particular, the eigenvalues of $\gamma P^{(\pi)}$ have norms bounded by $\gamma$, so the eigenvalues of $I-\gamma P^{(\pi)}$ have norms bounded below by $1-\gamma$; therefore, the eigenvalues of $I-\gamma P^{(\pi)}$ are nonzero, so it is invertible. As a consequence, we can solve for $d$ in (\ref{eqn:alg1proof:dformula3}) to get
\begin{align*}
d=(1-\gamma)(I-\gamma P^{(\pi)})^{-1}D.
\end{align*}
Finally, from (\ref{eqn:alg1proof:dformula2}) in Step 1 of this proof, we established that $D^{(\pi)}$ similarly satisfies
\begin{align*}
D^{(\pi)}=(1-\gamma)D+\gamma P^{(\pi)}D^{(\pi)}.
\end{align*}
As before, since $I-\gamma P^{(\pi)}$ is invertible, we have
\begin{align*}
D^{(\pi)}=(1-\gamma)(I-\gamma P^{(\pi)})^{-1}D=d.
\end{align*}
Thus,
\begin{align*}
\lambda_{s,a}=d_s\pi_{s,a}=D_s\pi_{s,a}=\Lambda^{(\pi)}_{s,a},
\end{align*}
so the first claim follows.

To see the second claim, note that since $\lambda$ is feasible, it must satisfy the second constraint of the LP. As shown in the first step of this proof, (\ref{eqn:alg1proof:secondconstraint2}) is equivalent to the demographic parity constraint. Thus, $\pi\in\Pi_{\text{DP}}$, as claimed.

To see the third claim, note that
\begin{align*}
R^{(\pi)}&=(1-\gamma)\mathbb{E}_{(s,a)\sim\Lambda^{(\pi)}}[R_{s,a}] \\
&=(1-\gamma)\sum_{s\in S}\sum_{a\in A}\Lambda^{(\pi)}_{s,a}R_{s,a}.
\end{align*}
In other words, the value of the objective of the LP for the point $\lambda$ is equal to $R^{(\pi)}$, as claimed.

\paragraph{Step 3.}

Finally, we use the results from the previous two steps to prove the theorem statement. First, let $\pi^*$ be the solution to (\ref{eqn:problem}). By the claim shown in the first step, $\Lambda^{(\pi^*)}$ is a feasible point of the LP in Algorithm~\ref{alg:modelbased}. Furthermore, by the claim shown in the second step, the value of the objective for $\lambda=\Lambda^{(\pi^*)}$ is $R^{(\pi^*)}$.

Next, let $\lambda^0$ be the solution to the LP in Algorithm~\ref{alg:modelbased}. By the claim shown in the second step, (i) $\lambda_0=\Lambda^{(\pi_0)}$, where $\pi_0$ is the policy returned by Algorithm~\ref{alg:modelbased}, (ii) $\pi_0\in\Pi_{\text{DP}}$, and (iii) the value of the objective for $\lambda^0$ is $R^{(\pi_0)}$.

It follows that $R^{(\pi^*)}\le R^{(\pi_0)}$, since $\pi_0$ maximizes the objective of the LP over feasible points (and $\Lambda^{(\pi^*)}$ is feasible). Since $\pi_0\in\Pi_{\text{DP}}$, it follows that $\pi_0$ is also a solution to (\ref{eqn:problem}). Thus, we have proven the theorem statement. $\qed$

\section{Proof of Theorem~\ref{thm:modelfree}}
\label{sec:modelfreeproof}

Our proof proceeds in three steps. First, we bound the error $|\tilde{\rho}^{(\pi)}-\rho^{(\pi)}|$ due to truncation. Second, we bound the estimation error $|\hat{\rho}^{(\pi)}-\tilde{\rho}^{(\pi)}|$. Third, we combine steps 1 and 2 to prove Theorem~\ref{thm:modelfree}.

\paragraph{Step 1.}

Note that for any policy $\pi$ and any $z\in Z$, we have
\begin{align*}
|\tilde{\rho}_z^{(\pi)}-\rho_z^{(\pi)}|
=\left|\sum_{t=T}^{\infty}\gamma^t\langle\rho_z,\Lambda^{(\pi,t)}\rangle\right|
&\le\sum_{t=T}^{\infty}\gamma^t\rho_{\text{max}} \\
&\le\frac{\gamma^T\rho_{\text{max}}}{1-\gamma} \\
&\le\frac{\sigma\epsilon}{4}.
\end{align*}

\paragraph{Step 2.}

For each $z\in Z$, let $\hat{\rho}_z^{(\pi)}$ be an estimate of $\tilde{\rho}_z^{(\pi)}$ using $m$ sampled rollouts $\zeta^{(1)},...,\zeta^{(m)}$. First, note that
\begin{align*}
|\hat{\rho}_z^{(\pi)}|\le\frac{\rho_{\text{max}}}{1-\gamma}
\end{align*}
is bounded, so we can apply Hoeffding's inequality (see Lemma~\ref{lem:hoeffding}) to get
\begin{align*}
\text{Pr}\left[|\hat{\rho}_z^{(\pi)}-\tilde{\rho}_z^{(\pi)}|\ge\frac{\sigma\epsilon}{4}\right]
&\le2\exp\left(-\frac{m\sigma^2\epsilon^2}{32\rho_{\text{max}}/(1-\gamma)}\right) \\
&\le\frac{\delta}{2}
\end{align*}
Since $Z=\{\text{maj},\text{min}\}$, by a union bound,
\begin{align*}
|\hat{\rho}_z^{(\pi)}-\tilde{\rho}_z^{(\pi)}|&\le\frac{\epsilon}{4}\hspace{0.2in}(\forall z\in Z)
\end{align*}
with probability at least $1-\delta$.

\paragraph{Step 3.}

Now, we can prove Theorem~\ref{thm:modelfree}. First, note that with probability $1-\delta$,
\begin{align*}
|\hat{\rho}_z^{(\pi)}-\rho_z^{(\pi)}|&\le|\hat{\rho}^{(\pi)}-\tilde{\rho}^{(\pi)}|+|\tilde{\rho}^{(\pi)}-R^{(\pi)}| \\
&\le\frac{\sigma\epsilon}{4}+\frac{\sigma^2\epsilon}{4} \\
&\le\frac{\sigma\epsilon}{2},
\end{align*}
for all $z\in Z$. Thus,
\begin{align*}
&|\rho_{\text{maj}}^{(\pi)}-\rho_{\text{min}}^{(\pi)}| \\
&\le|\rho_{\text{maj}}^{(\pi)}-\hat{\rho}_{\text{maj}}^{(\pi)}|+|\hat{\rho}_{\text{maj}}^{(\pi)}-\hat{\rho}_{\text{min}}^{(\pi)}|+|\hat{\rho}_{\text{min}}^{(\pi)}-\rho_{\text{min}}^{(\pi)}| \\
&\le\frac{\sigma\epsilon}{2}+(1-\sigma)\epsilon+\frac{\sigma\epsilon}{2} \\
&=\epsilon,
\end{align*}
which implies that $\pi\in\Pi_{\text{DP},\epsilon}$. Thus, the theorem follows. $\qed$

\section{Proof of Theorem~\ref{thm:unknowntrans}}
\label{sec:thmunknowntransproof}

We prove the following lemma; Theorem~\ref{thm:unknowntrans} follows by choosing $\epsilon'=N^{-1/3}$.
\begin{lemma}
\label{lem:unknowntrans}
Let $\epsilon,\epsilon',\delta\in\mathbb{R}_+$ be given. Assume that $R_{\text{max}}$ be an upper bound on $R$ (i.e., $\|R\|_{\infty}=R_{\text{max}}$) and on $\rho$. Let $\tilde{\epsilon}=\min\{\epsilon,\epsilon'\}$, and let
\begin{align*}
N_0&=\frac{128T^4\cdot|S|^2\cdot R_{\text{max}}^2\cdot\log(2|S|^2|A|/\delta)}{\lambda_0^2\tilde{\epsilon}^2}.
\end{align*}
Let $\hat{M}=(S,A,D,\hat{P},R,T)$, and $\hat{\pi}$ be the optimal policy for $\hat{M}$ in $\hat{\Pi}_{\text{DP},\epsilon/2}$ (i.e., the set of policies satisfying demographic parity for $\hat{M}$). Let $M=(S,A,D,P,R,T)$, and $\pi^*$ be optimal for $M$ in $\Pi_{\text{DP},\epsilon/4}$. Then, $\hat{\pi}\in\Pi_{\text{DP},\epsilon}$, and $R^{(\pi^*)}-R^{(\hat{\pi})}\le\epsilon'$, where $R^{(\pi)}$ is defined for $M$.
\end{lemma}

Our proof proceeds in three steps. First, we prove that for any $\epsilon_0,\delta_0$, we can choose $N_0$ sufficiently large so that
\begin{align*}
\|P-\hat{P}\|_{\infty}\le\epsilon_0
\end{align*}
with probability at least $1-\delta_0$. Second, we prove that assuming $\|P-\hat{P}\|_{\infty}\le\epsilon_0$, then for any policy $\pi$, we have
\begin{align*}
|R^{(\pi)}-\hat{R}^{(\pi)}|\le T^2\cdot|S|\cdot R_{\text{max}}\cdot\epsilon_0,
\end{align*}
where $R^{(\pi)}$ (resp., $\hat{R}^{(\pi)}$) is the expected cumulative distribution assuming the transitions are $P$ (resp., $\hat{P}$), and similarly for the agent rewards $\rho$. Third, we use the first two steps to prove the lemma statement.

\paragraph{Step 1.}

Given $\epsilon_0,\delta_0\in\mathbb{R}_+$, we claim that for
\begin{align*}
N_0=\frac{2\log(2|S|^2|A|/\delta_0)}{\lambda_0^2\epsilon_0^2},
\end{align*}
then our estimate $\hat{P}$ satisfies
\begin{align*}
\|\hat{P}-P\|_{\infty}\le\epsilon_0
\end{align*}
with probability at least $1-\delta_0$.

Let $I_{s,a}$ be the random variable indicating whether our algorithm observes a tuple $(s,a,s')$ (for some $s'\in S$) on a single episode, and let $I_{s,a,i}$ be samples of $I_{s,a}$ for each of the $N_0$ exploratory episodes taken by our algorithm. Let
\begin{align*}
\mu_{s,a}^{(I)}&=\mathbb{E}[I_{s,a}] \\
\hat{\mu}_{s,a}^{(I)}&=\frac{1}{N_0}\sum_{i=1}^{N_0}I_{s,a,i}.
\end{align*}
Then, by Hoeffding's inequality (see Lemma~\ref{lem:hoeffding}), we have
\begin{align}
\label{eqn:thmunknowntrans:step1:1}
\text{Pr}\left[|\hat{\mu}_{s,a}^{(I)}-\mu_{s,a}^{(I)}|\ge\epsilon\right]\le2e^{-2N_0\epsilon^2}.
\end{align}
By assumption, we have
\begin{align*}
\mu_{s,a}^{(I)}=\Lambda^{(\pi_0)}_{s,a}\ge\lambda_0,
\end{align*}
so using $\epsilon=\lambda_0/2$ in (\ref{eqn:thmunknowntrans:step1:1}), we have
\begin{align}
\label{eqn:thmunknowntrans:step1:2}
\hat{\mu}_{s,a}^{(I)}\ge\frac{\mu_{s,a}^{(I)}}{2}\ge\frac{\lambda_0}{2}
\end{align}
with probability at least
\begin{align*}
1-2e^{-N_0(\mu_{s,a}^{(I)})^2/2}\ge1-2e^{-N_0\lambda_0^2/2}.
\end{align*}
Taking a union bound over $s\in S$ and $a\in A$, we have (\ref{eqn:thmunknowntrans:step1:2}) holds for every $s\in S$ and $a\in A$ with probability at least
\begin{align}
\label{eqn:thmunknowntrans:step1:3}
1-2|S|\cdot|A|\cdot e^{-N_0\lambda_0^2/2}.
\end{align}
In this event, we have at least $\frac{N_0\lambda_0}{2}$ observations $(s,a,s')$ (for some $s'\in S$) for every $s\in S$ and $a\in A$.

Now, for an observation $(s,a,s'')$, let $J_{s,a,s'}$ be the random variable indication whether $s'=s''$. Without loss of generality, we assume that we have exactly $N_1=\frac{N_0\lambda_0}{2}$ samples $J_{s,a,s',j}$ of $J_{s,a,s'}$ for each $s\in S$ and $a\in A$. Let
\begin{align*}
\mu_{s,a,s'}^{(J)}&=\mathbb{E}[J_{s,a,s'}] \\
\hat{\mu}_{s,a,s'}^{(J)}&=\frac{1}{N_1}\sum_{j=1}^{N_1}J_{s,a,s',j}.
\end{align*}
Then, by Hoeffding's inequality (see Lemma~\ref{lem:hoeffding}), we have
\begin{align}
\label{eqn:thmunknowntrans:step1:4}
\text{Pr}\left[|\hat{\mu}_{s,a,s'}^{(J)}-\mu_{s,a,s'}|\ge\epsilon\right]\le2e^{-2N_1\epsilon^2}.
\end{align}
Note that by definition, $\mu_{s,a,s'}^{(J)}=P_{s,a,s'}$ and $\hat{\mu}_{s,a,s'}^{(J)}=\hat{P}_{s,a,s'}$. Thus, taking $\epsilon=\epsilon_0$ in (\ref{eqn:thmunknowntrans:step1:4}), we have
\begin{align}
\label{eqn:thmunknowntrans:step1:5}
|P_{s,a,s'}-\hat{P}_{s,a,s'}|\le\epsilon_0
\end{align}
with probability at least
\begin{align*}
1-2e^{-2N_1\epsilon_0^2}.
\end{align*}
Taking a union bound over all $s,s'\in S$ and $a\in A$, we have (\ref{eqn:thmunknowntrans:step1:5}) for all $s,s'\in S$ and $a\in A$ with probability at least
\begin{align}
\label{eqn:thmunknowntrans:step1:6}
1-2|S|^2|A|\cdot e^{-2N_1\epsilon_0^2}.
\end{align}
In other words, in this event, we have $\|P-\hat{P}\|_{\infty}\le\epsilon_0$.

Taking a union bound over (\ref{eqn:thmunknowntrans:step1:3}) and (\ref{eqn:thmunknowntrans:step1:6}), we have
\begin{align*}
\|P-\hat{P}\|_{\infty}\le\epsilon_0
\end{align*}
with probability at least
\begin{align*}
&1-2|S|^2|A|\cdot e^{-2N_1\epsilon_0^2}-2|S|\cdot|A|\cdot e^{-N_0\lambda_0^2/2} \\
&=1-2|S|^2|A|\cdot e^{-N_0\lambda_0\epsilon_0^2}-2|S|\cdot|A|\cdot e^{-N_0\lambda_0^2/2} \\
&\ge1-2|S|^2|A|\cdot e^{-N_0\lambda_0^2\epsilon_0^2/2} \\
&=\delta_0,
\end{align*}
as claimed.

\paragraph{Step 2.}

We claim that assuming
\begin{align*}
\|P-\hat{P}\|_{\infty}\le\epsilon_0,
\end{align*}
then for any policy $\pi$, we have
\begin{align*}
|R^{(\pi)}-\hat{R}^{(\pi)}|\le T^2\cdot|S|\cdot R_{\text{max}}\cdot\epsilon_0,
\end{align*}
where $R^{(\pi)}$ is the expected cumulative reward for $\pi$ in the MDP $M=(S,A,D,P,R,T)$ and $\hat{R}^{(\pi)}$ is the expected cumulative reward for $\pi$ in the MDP $\hat{M}=(S,A,D,\hat{P},R,T)$. Note that we have replaced the discount factor $\gamma$ with the time horizon $T$. In addition, for all $z\in Z$, we have
\begin{align*}
|\rho_z^{(\pi)}-\hat{\rho}_z^{(\pi)}|\le T\cdot|S|\cdot R_{\text{max}}\cdot\epsilon_0,
\end{align*}
where
\begin{align*}
\rho_z^{(\pi)}=\mathbb{E}_{(s,a)\sim\Lambda_z^{(\pi)}}[\rho_{s,a}],
\end{align*}
is the expected cumulative agent reward for the MDP $M$, and $\hat{\rho}_z^{(\pi)}$ is the expected cumulative agent reward for the MDP $\hat{M}$. We only prove the claim for $|R^{(\pi)}-\hat{R}^{(\pi)}|$; the claim for $|\rho_z^{(\pi)}-\hat{\rho}_z^{(\pi)}|$ follows using the same argument.

Let $W\in\mathbb{R}^{|S|}$ be
\begin{align*}
W_s=\langle\pi_{s,\cdot},R_{s,\cdot}\rangle=\sum_{a\in A}\pi_{s,a}R_{s,a}.
\end{align*}
Then, we have
\begin{align*}
R^{(\pi)}
&=\langle R,\Lambda^{(\pi)}\rangle \\
&=\sum_{s\in S}\sum_{a\in A}D^{(\pi)}_s\pi_{s,a}R_{s,a} \\
&=\langle D^{(\pi)},W\rangle.
\end{align*}
Now, note that
\begin{align*}
D^{(\pi,t)}=(P^{(\pi)})D,
\end{align*}
so we have
\begin{align*}
D^{(\pi)}
&=\frac{1}{T}\sum_{t=0}^{T-1}D^{(\pi,t)} \\
&=\frac{1}{T}\left[\sum_{t=0}^{T-1}(P^{(\pi)})^tD\right].
\end{align*}
Thus,
\begin{align*}
R^{(\pi)}=\sum_{t=0}^{T-1}\left\langle(P^{(\pi)})^tD,W\right\rangle.
\end{align*}
Similarly,
\begin{align*}
\hat{R}^{(\pi)}=\sum_{t=0}^{T-1}\left\langle(\hat{P}^{(\pi)})^tD,W\right\rangle.
\end{align*}
It follows that
\begin{align*}
R^{(\pi)}-\hat{R}^{(\pi)}
&=\sum_{t=0}^{T-1}\left\langle(P^{(\pi)})^tD-(\hat{P}^{(\pi)})^tD,W\right\rangle.
\end{align*}
Thus,
\begin{align}
\label{eqn:hatP:rewardbound1}
&|R^{(\pi)}-\hat{R}^{(\pi)}| \nonumber \\
&\le\sum_{t=0}^{T-1}\|(P^{(\pi)})^tD-(\hat{P}^{(\pi)})^tD\|_{\infty}\cdot\|W\|_1 \nonumber \\
&\le\sum_{t=1}^{T-1}\epsilon_0\cdot T\cdot\|D\|_{\infty}\cdot\|W\|_1,
\end{align}
where the first line follows from H\"{o}lder's inequality, and the second line follows from properties of the matrix norm, from the fact that
\begin{align*}
&\|(P^{(\pi)})^t-(\hat{P}^{(\pi)})^t\|_{\infty} \\
&=\|P^{(\pi)}-\hat{P}^{(\pi)}\|_{\infty}\cdot\sum_{s=0}^{t-1}\|P^{(\pi)}\|_{\infty}^s\cdot\|\hat{P}^{(\pi)}\|_{\infty}^{t-s-1} \\
&\le\epsilon_0\cdot T,
\end{align*}
and using the fact that the summand is zero for $t=0$ since $(P^{(\pi)})^0=(\hat{P}^{(\pi)})^0=I$. Note that
\begin{align}
\label{eqn:hatP:transitionbound}
\|D\|_{\infty}\le1
\end{align}
Furthermore,
\begin{align*}
|W_s|=|\langle\pi_{s,\cdot},R_{s,\cdot}\rangle|
&\le\|\pi_{s,\cdot}\|_1\cdot\|R_{s,\cdot}\|_{\infty} \\
&\le\|R_{s,\cdot}\|_{\infty} \\
&\le R_{\text{max}},
\end{align*}
where the first inequality follows from H\"{o}lder's inequality and the second inequality follows since $\pi_{s,\cdot}$ is a discrete probability distribution. Therefore,
\begin{align}
\label{eqn:hatP:wbound}
\|W\|_1=\sum_{s\in S}|W_s|\le|S|\cdot R_{\text{max}}.
\end{align}
Plugging (\ref{eqn:hatP:transitionbound}) and (\ref{eqn:hatP:wbound}) into (\ref{eqn:hatP:rewardbound1}) gives
\begin{align*}
|R^{(\pi)}-\hat{R}^{(\pi)}|
\le T^2\cdot|S|\cdot R_{\text{max}}\cdot\epsilon_0.
\end{align*}

\paragraph{Step 3.}

Now, we prove the theorem. Let $\hat{\pi}$ be the optimal policy for $\hat{M}$ (i.e., transitions $\hat{P}$) satisfying $\hat{\pi}\in\hat{\Pi}_{\text{DP},\epsilon/2}$. Similarly, let $\pi^*$ be the optimal policy for $M$ (i.e., transitions $P$) satisfying $\pi^*\in\Pi_{\text{DP},\epsilon/4}$. We apply the second step with
\begin{align*}
\epsilon_0&=\frac{\tilde{\epsilon}}{8T^2\cdot|S|\cdot R_{\text{max}}} \\
\delta_0&=\delta.
\end{align*}
Then, by the first step, for all $z,z'\in Z$, we have
\begin{align*}
&\rho_z^{(\hat{\pi})}-\rho_{z'}^{(\hat{\pi})} \\
&\le(\rho_z^{(\hat{\pi})}-\hat{\rho}_z^{(\hat{\pi})})
+(\hat{\rho}_z^{(\hat{\pi})}-\hat{\rho}_{z'}^{(\hat{\pi})})+
+(\rho_{z'}^{(\hat{\pi})}-\hat{\rho}_{z'}^{(\hat{\pi})}) \\
&\le T^2\cdot|S|\cdot R_{\text{max}}\cdot\epsilon_0+\frac{\epsilon}{2}+T^2\cdot|S|\cdot R_{\text{max}}\cdot\epsilon_0 \\
&\le\epsilon,
\end{align*}
where the inequality on the third line follows because $\hat{\pi}\in\hat{\Pi}_{\text{DP},\epsilon/2}$, and the inequality on the last line follows since $\tilde{\epsilon}\le\epsilon$. Thus, we guarantee that $\hat{\pi}\in\Pi_{\text{DP},\epsilon}$.

Next, note that similarly, for all $z,z'\in Z$, we have
\begin{align*}
\hat{\rho}_z^{(\pi^*)}-\hat{\rho}_{z'}^{(\pi^*)}\le\frac{\epsilon}{2},
\end{align*}
so $\pi^*\in\hat{\Pi}_{\text{DP},\epsilon/2}$. As a consequence, we have
\begin{align*}
&R^{(\pi^*)}-R^{(\hat{\pi})} \\
&=(R^{(\pi^*)}-\hat{R}^{(\pi^*)})+(\hat{R}^{(\pi^*)}-\hat{R}^{(\hat{\pi})})+(\hat{R}^{(\hat{\pi})}-R^{(\hat{\pi})}) \\
&\le T^2\cdot|S|\cdot R_{\text{max}}\cdot\epsilon_0+0+T^2\cdot|S|\cdot R_{\text{max}}\cdot\epsilon_0 \\
&\le\epsilon',
\end{align*}
where the inequality on the third line follows because $\hat{\pi}$ maximizes $\hat{R}^{(\pi)}$ over $\pi\in\hat{\Pi}_{\text{DP},\epsilon/2}$ and $\pi^*\in\hat{\Pi}_{\text{DP},\epsilon/2}$, and the inequality on the last line follows since $\tilde{\epsilon}\le\epsilon'$.

Thus, the lemma statement follows. $\qed$

\section{Technical Lemmas}

\begin{lemma}
\label{lem:hoeffding}
(Hoeffding's inequality) Let $X\sim p_X$ be a random variable with domain $[a,b]\subseteq\mathbb{R}$ and mean $\mu_X$, and let $\hat{\mu}_X=n^{-1}\sum_{i=1}^nX_i$ be an estimate of $\mu_X$ a using $n$ i.i.d. samples $X_i\sim p_X$. Then, we have
\begin{align}
\text{Pr}\left[|\hat{\mu}_X-\mu_X|\ge\epsilon\right]\le2\exp\left(-\frac{2n\epsilon^2}{(b-a)^2}\right),
\end{align}
where the probability is taken over the randomness in the i.i.d. samples $X_1,...,X_n\sim p_X$.
\end{lemma}

\begin{proof}
See~\cite{wainwright2019high} for a proof. 
\end{proof}

\section{Experimental Details \& Additional Results}
\label{sec:expappendix}

\begin{figure*}[t]
\centering
\begin{tabular}{ccc}
\includegraphics[width=0.27\textwidth]{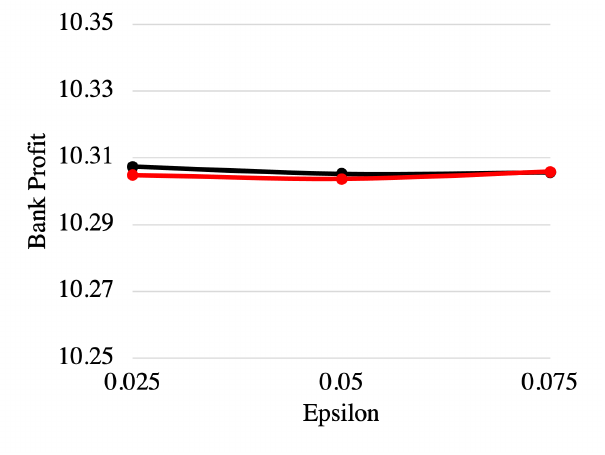} &
\includegraphics[width=0.27\textwidth]{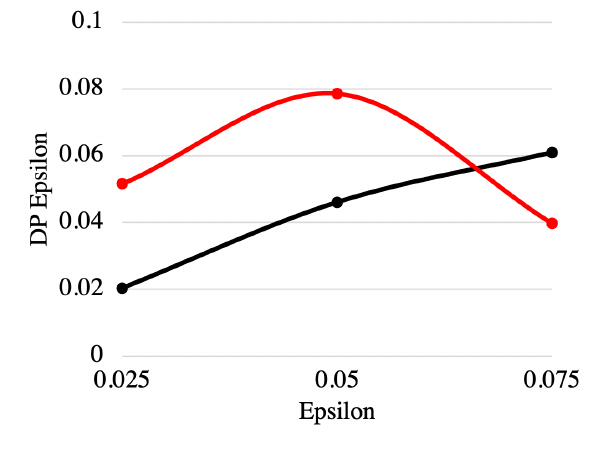} &
\includegraphics[width=0.34\textwidth]{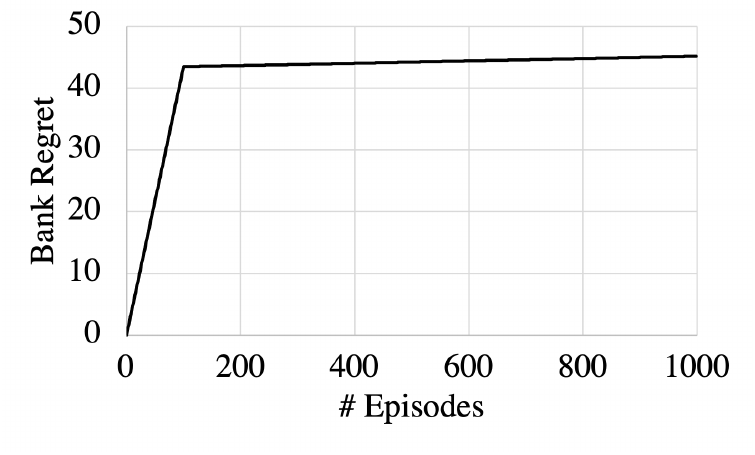} \\
(a) & (b) & (c)
\end{tabular}
\caption{Demographic parity (a) objective value, (b) constraint value for our algorithm (black) and the optimistic baseline (red). (c) Regret of our reinforcement learning algorithm.}
\label{fig:appendixexp}
\end{figure*}

\textbf{Parameters.}
We use the following parameters for our loan MDP:
\begin{align*}
I&=0.17318629 \\
p_Z&=0.29294318 \\
\alpha_{\text{maj}}&=0.65338681 \\
\beta_{\text{maj}}&=0.20783559 \\
\alpha_{\text{min}}&=0.48824268 \\
\beta_{\text{min}}&=0.48346869 \\
\lambda&=0.01 \\
\tau&=0.1 \\
\epsilon&=0.1 \\
T&=50 \\
T_{\text{maj}}&=10 \\
T_{\text{min}}&=7.
\end{align*}

\textbf{Additional results for Algorithm~\ref{alg:modelfree}.}
We additionally study how Algorithm~\ref{alg:modelfree} varies with the fairness constraint threshold $\epsilon$. In Figure~\ref{fig:appendixexp} (a,b), we show the objective value achieved and the fairness constraint value achieved by our algorithm and the optimistic algorithm for the demographic parity constraint. While the objective values achieved are very similar, the optimistic algorithm does not always satisfy the fairness constraint. In particular, for $\epsilon=0.025$, its constraint value is $0.052$ (exceeds $\epsilon$ by 108\%), and for $\epsilon=0.05$, it is $0.079$ (exceeds $\epsilon$ by 59\%). Intuitively, there are multiple policies that achieve the same objective value, but the optimistic algorithm sometimes fails to find the ones that are fair. In contrast, our algorithm always satisfies the fairness constraint.

\textbf{Results for Algorithm~\ref{alg:modelbased}.}
We have evaluated Algorithm~\ref{alg:modelbased} on a modified version of our loan MDP where $\alpha$ and $\beta$ are discretized and thresholded to make the state space finite. For this MDP, we have compared Algorithm~\ref{alg:modelbased} to solving an unconstrained MDP---i.e., without the demographic parity fairness constraint. We use $\epsilon = 0.01$. Our results are as follows:
(i) for Algorithm~\ref{alg:modelbased}, the cumulative expected reward is $0.68$ and the fairness constraint value is $0.01$, and
(ii) for the unconstrained algorithm, the cumulative expected reward is $0.69$ and the fairness constraint value is $0.26$.
In other words, for a small reduction in reward, our algorithm substantially improves fairness. The remaining baselines cannot be implemented using the approach in Algorithm~\ref{alg:modelbased}.

\textbf{Results for reinforcement learning.}
We have run our reinforcement learning algorithm in conjunction with the
We run the algorithm for $N=1000$ episodes total. In particular, we explore for $100$ episodes using a conservative policy $\pi_0$ that ignores the state; then, we use the estimated transitions to learn the optimal policy $\hat{\pi}$ and use $\hat{\pi}$ for the remaining $900$ episodes. Since our model is parameterized by $\alpha$ and $\beta$, we estimate these quantities instead of directly estimating the transitions. We show the regret compared to the optimal policy in Figure~\ref{fig:appendixexp} (c), averaged over 5 iterations. As can be seen, the regret quickly increases while using $\pi_0$, and then becomes almost flat when using $\hat{\pi}$. We note that our algorithm satisfies the fairness constraint across all episodes and iterations.

%% file: paper.bbl
\begin{thebibliography}{25}
\providecommand{\natexlab}[1]{#1}
\providecommand{\url}[1]{\texttt{#1}}
\expandafter\ifx\csname urlstyle\endcsname\relax
  \providecommand{\doi}[1]{doi: #1}\else
  \providecommand{\doi}{doi: \begingroup \urlstyle{rm}\Url}\fi

\bibitem[Achiam et~al.(2017)Achiam, Held, Tamar, and
  Abbeel]{achiam2017constrained}
Joshua Achiam, David Held, Aviv Tamar, and Pieter Abbeel.
\newblock Constrained policy optimization.
\newblock In \emph{ICML}, 2017.

\bibitem[Altman(1999)]{altman1999constrained}
Eitan Altman.
\newblock \emph{Constrained Markov decision processes}, volume~7.
\newblock CRC Press, 1999.

\bibitem[Awasthi et~al.(2020)Awasthi, Cortes, Mansour, and
  Mohri]{awasthi2020beyond}
Pranjal Awasthi, Corinna Cortes, Yishay Mansour, and Mehryar Mohri.
\newblock Beyond individual and group fairness.
\newblock \emph{arXiv preprint arXiv:2008.09490}, 2020.

\bibitem[Bechavod et~al.(2019)Bechavod, Ligett, Roth, Waggoner, and
  Wu]{bechavod2019equal}
Yahav Bechavod, Katrina Ligett, Aaron Roth, Bo~Waggoner, and Zhiwei~Steven Wu.
\newblock Equal opportunity in online classification with partial feedback.
\newblock \emph{arXiv preprint arXiv:1902.02242}, 2019.

\bibitem[Calders et~al.(2009)Calders, Kamiran, and
  Pechenizkiy]{calders2009building}
Toon Calders, Faisal Kamiran, and Mykola Pechenizkiy.
\newblock Building classifiers with independency constraints.
\newblock In \emph{Data mining workshops, 2009. ICDMW'09. IEEE international
  conference on}, pages 13--18. IEEE, 2009.

\bibitem[Creager et~al.(2019)Creager, Madras, Pitassi, and
  Zemel]{creager2019causal}
Elliot Creager, David Madras, Toniann Pitassi, and Richard Zemel.
\newblock Causal modeling for fairness in dynamical systems.
\newblock \emph{arXiv preprint arXiv:1909.09141}, 2019.

\bibitem[D'Amour et~al.(2020)D'Amour, Srinivasan, Atwood, Baljekar, Sculley,
  and Halpern]{d2020fairness}
Alexander D'Amour, Hansa Srinivasan, James Atwood, Pallavi Baljekar, D~Sculley,
  and Yoni Halpern.
\newblock Fairness is not static: deeper understanding of long term fairness
  via simulation studies.
\newblock In \emph{Proceedings of the 2020 Conference on Fairness,
  Accountability, and Transparency}, pages 525--534, 2020.

\bibitem[Dwork et~al.(2012)Dwork, Hardt, Pitassi, Reingold, and
  Zemel]{dwork2012fairness}
Cynthia Dwork, Moritz Hardt, Toniann Pitassi, Omer Reingold, and Richard Zemel.
\newblock Fairness through awareness.
\newblock In \emph{Proceedings of the 3rd innovations in theoretical computer
  science conference}, pages 214--226. ACM, 2012.

\bibitem[Elzayn et~al.(2019)Elzayn, Jabbari, Jung, Kearns, Neel, Roth, and
  Schutzman]{elzayn2018fair}
Hadi Elzayn, Shahin Jabbari, Christopher Jung, Michael Kearns, Seth Neel, Aaron
  Roth, and Zachary Schutzman.
\newblock Fair algorithms for learning in allocation problems.
\newblock 2019.

\bibitem[Hardt et~al.(2016)Hardt, Price, Srebro, et~al.]{hardt2016equality}
Moritz Hardt, Eric Price, Nati Srebro, et~al.
\newblock Equality of opportunity in supervised learning.
\newblock In \emph{Advances in neural information processing systems}, pages
  3315--3323, 2016.

\bibitem[Hashimoto et~al.(2018)Hashimoto, Srivastava, Namkoong, and
  Liang]{hashimoto2018fairness}
Tatsunori~B Hashimoto, Megha Srivastava, Hongseok Namkoong, and Percy Liang.
\newblock Fairness without demographics in repeated loss minimization.
\newblock In \emph{ICML}, 2018.

\bibitem[Hu et~al.(2012)Hu, Hu, and Chang]{hu2012stochastic}
Jiaqiao Hu, Ping Hu, and Hyeong~Soo Chang.
\newblock A stochastic approximation framework for a class of randomized
  optimization algorithms.
\newblock \emph{IEEE Transactions on Automatic Control}, 57\penalty0
  (1):\penalty0 165--178, 2012.

\bibitem[Jabbari et~al.(2017)Jabbari, Joseph, Kearns, Morgenstern, and
  Roth]{jabbari2017fairness}
Shahin Jabbari, Matthew Joseph, Michael Kearns, Jamie Morgenstern, and Aaron
  Roth.
\newblock Fairness in reinforcement learning.
\newblock In \emph{Proceedings of the 34th International Conference on Machine
  Learning-Volume 70}, pages 1617--1626. JMLR. org, 2017.

\bibitem[Joseph et~al.(2016)Joseph, Kearns, Morgenstern, and
  Roth]{joseph2016fairness}
Matthew Joseph, Michael Kearns, Jamie~H Morgenstern, and Aaron Roth.
\newblock Fairness in learning: Classic and contextual bandits.
\newblock In \emph{Advances in Neural Information Processing Systems}, pages
  325--333, 2016.

\bibitem[Kilbertus et~al.(2017)Kilbertus, Carulla, Parascandolo, Hardt,
  Janzing, and Sch{\"o}lkopf]{kilbertus2017avoiding}
Niki Kilbertus, Mateo~Rojas Carulla, Giambattista Parascandolo, Moritz Hardt,
  Dominik Janzing, and Bernhard Sch{\"o}lkopf.
\newblock Avoiding discrimination through causal reasoning.
\newblock In \emph{Advances in Neural Information Processing Systems}, pages
  656--666, 2017.

\bibitem[Kilbertus et~al.(2019)Kilbertus, Gomez-Rodriguez, Sch{\"o}lkopf,
  Muandet, and Valera]{kilbertus2019fair}
Niki Kilbertus, Manuel Gomez-Rodriguez, Bernhard Sch{\"o}lkopf, Krikamol
  Muandet, and Isabel Valera.
\newblock Fair decisions despite imperfect predictions.
\newblock AISTATS, 2019.

\bibitem[Kusner et~al.(2017)Kusner, Loftus, Russell, and
  Silva]{kusner2017counterfactual}
Matt~J Kusner, Joshua Loftus, Chris Russell, and Ricardo Silva.
\newblock Counterfactual fairness.
\newblock In \emph{Advances in Neural Information Processing Systems}, pages
  4066--4076, 2017.

\bibitem[Lakkaraju et~al.(2017)Lakkaraju, Kleinberg, Leskovec, Ludwig, and
  Mullainathan]{lakkaraju2017selective}
Himabindu Lakkaraju, Jon Kleinberg, Jure Leskovec, Jens Ludwig, and Sendhil
  Mullainathan.
\newblock The selective labels problem: Evaluating algorithmic predictions in
  the presence of unobservables.
\newblock In \emph{Proceedings of the 23rd ACM SIGKDD International Conference
  on Knowledge Discovery and Data Mining}, pages 275--284. ACM, 2017.

\bibitem[Lattimore and Szepesv{\'a}ri()]{lattimore2018bandit}
Tor Lattimore and Csaba Szepesv{\'a}ri.
\newblock Bandit algorithms.

\bibitem[Liu et~al.(2018)Liu, Dean, Rolf, Simchowitz, and
  Hardt]{liu2018delayed}
Lydia~T Liu, Sarah Dean, Esther Rolf, Max Simchowitz, and Moritz Hardt.
\newblock Delayed impact of fair machine learning.
\newblock In \emph{ICML}, 2018.

\bibitem[Mannor et~al.(2003)Mannor, Rubinstein, and Gat]{mannor2003cross}
Shie Mannor, Reuven~Y Rubinstein, and Yohai Gat.
\newblock The cross entropy method for fast policy search.
\newblock In \emph{Proceedings of the 20th International Conference on Machine
  Learning (ICML-03)}, pages 512--519, 2003.

\bibitem[Nabi and Shpitser(2018)]{nabi2018fair}
Razieh Nabi and Ilya Shpitser.
\newblock Fair inference on outcomes.
\newblock In \emph{Proceedings of the... AAAI Conference on Artificial
  Intelligence. AAAI Conference on Artificial Intelligence}, volume 2018, page
  1931. NIH Public Access, 2018.

\bibitem[Sutton and Barto(2018)]{sutton2018reinforcement}
Richard~S Sutton and Andrew~G Barto.
\newblock \emph{Reinforcement learning: An introduction}.
\newblock MIT press, 2018.

\bibitem[Wainwright(2019)]{wainwright2019high}
Martin~J. Wainwright.
\newblock \emph{High-dimensional statistics: A non-asymptotic viewpoint}.
\newblock Cambridge University Press, 2019.

\bibitem[Wen and Topcu(2018)]{wen2018constrained}
Min Wen and Ufuk Topcu.
\newblock Constrained cross-entropy method for safe reinforcement learning.
\newblock In \emph{Advances in Neural Information Processing Systems}, pages
  7461--7471, 2018.

\end{thebibliography}
